\documentclass{article}

\usepackage[accepted]{aistats2021}
\usepackage[round]{natbib}

\usepackage[utf8]{inputenc} 
\usepackage[T1]{fontenc}    
\usepackage{hyperref}       
\usepackage{url}            
\usepackage{booktabs}       
\usepackage{amsfonts}       
\usepackage{nicefrac}       
\usepackage{microtype}      
\usepackage{amsmath,amssymb}
\usepackage{amsthm}
\usepackage{times}
\usepackage{multicol} 
\usepackage{bm}
\usepackage{framed}
\usepackage{graphicx} 
\usepackage{subfigure} 
\usepackage{subfigmat}

\usepackage{xr}
 
\makeatletter
\newcommand*{\addFileDependency}[1]{
  \typeout{(#1)}
  \@addtofilelist{#1}
  \IfFileExists{#1}{}{\typeout{No file #1.}}
}
\makeatother
 
\newcommand*{\myexternaldocument}[1]{%
    \externaldocument{#1}%
    \addFileDependency{#1.tex}%
    \addFileDependency{#1.aux}%
}
\myexternaldocument{main}

\usepackage{algorithmic}
\usepackage[boxruled]{algorithm2e}

\usepackage{color}
\usepackage{array}

\usepackage[symbol]{footmisc}

\theoremstyle{theorem}
\newtheorem{theorem}{Theorem}[section]
\newtheorem{lemma}[theorem]{Lemma}
\newtheorem{proposition}[theorem]{Proposition}

\newtheorem*{example}{Example}
\theoremstyle{definition}
\newtheorem{definition}{Definition}[section]
\newtheorem{assumption}{Assumption}
\theoremstyle{remark}
\newtheorem*{remark}{Remark}

\begin{document}
\twocolumn[
\aistatstitle{Gradient Descent in RKHS with Importance Labeling}

\aistatsauthor{ Tomoya Murata\And Taiji Suzuki}

\aistatsaddress{ NTT DATA Mathematical Systems Inc. \\The University of Tokyo\footnotemark[1]\\murata@msi.co.jp \And The University of Tokyo\footnotemark[1]\\RIKEN AIP\footnotemark[2]\\taiji@mist.i.u-tokyo.ac.jp} ]

\begin{abstract}
Labeling cost is often expensive and is a fundamental limitation of supervised learning. In this paper, we study importance labeling problem, in which we are given many unlabeled data and select a limited number of data to be labeled from the unlabeled data, and then a learning algorithm is executed on the selected one. We propose a new importance labeling scheme that can effectively select an informative subset of unlabeled data in least squares regression in Reproducing Kernel Hilbert Spaces (RKHS). We analyze the generalization error of gradient descent combined with our labeling scheme and show that the proposed algorithm achieves the optimal rate of convergence in much wider settings and especially gives much better generalization ability in a small label noise setting than the usual uniform sampling scheme. Numerical experiments verify our theoretical findings.
\end{abstract}

\footnotetext[1]{Department of Mathematical Informatics, Graduate School of Information Science and Technology, The University of Tokyo}
\footnotetext[2]{Center for Advanced Intelligence Project, RIKEN, Tokyo, Japan}
\addtocounter{footnote}{2}
\section{Introduction}
One of the most popular task in machine learning is supervised learning, in which we estimate a function that maps an input to its label based on finite labeled examples called training data. The goodness of the learned function is measured by the generalization ability, that is roughly the accuracy of the learned function for previously unseen data. Statistical learning theory is a powerful tool which gives a framework for analysing the generalization errors of learning algorithms \citep{vapnik1998statistical}. Enormous learning algorithms have been proposed and their generalization abilities are analysed in various settings. \par
In spite of the great successes of supervised learning, it has a fundamental limitation due to the expensive cost for making training examples. Particularly, it is often the case that collecting input data is cheap but to give labels of them is limited or expensive and that is one of bottlenecks in supervised learning \citep{roh2019survey}. The dilemma is that the more labeled data, better generalization ability is guaranteed but the higher labeling cost is incurred. \par
In this limited situation, {\it{importance labeling}} problem naturally arises, which is a special case of active learning \citep{settles2009active}. In the importance labeling settings, we first collect many unlabeled examples. Then we choose a limited number of examples to be labeled from unlabeled ones. The most naive selection of labeled examples is based on uniform subsampling from unlabeled data. What we expect here is that if we choose labeled samples effectively, then better generalization ability may be acquired.
\par
Despite the significance of the problem, theoretical aspects of importance labeling is little known. The essential question is what importance labeling scheme surpasses the standard uniform labeling in what settings. \par

In this paper, we consider this quite general question in the context of least squares  regression in Reproducing Kernel Hilbert Spaces (RKHS). Kernel method is classical and promising approach for learning nonlinear functions \citep{scholkopf2002learning}. In kernel method, input data is mapped to an (potentially) infinite dimensional feature space and then a linear predictor on the feature space is learned. The feature space is determined by the user-defined kernel function and numerous kernel functions are known, e.g., classical Gaussian kernel and more modern neural tangent kernel (NTK) \citep{jacot2018neural}. Least squares regression in RKHS has a long history and its generalization ability has been thoroughly studied in supervised learning settings \citep{caponnetto2007optimal,steinwart2009optimal,rosasco2015learning,dieuleveut2016nonparametric,rudi2017generalization}. However, these papers do not consider the utilization of the unlabeled data and hence the derived theoretical generalization ability may be sub-optimal because the uniform labeling never captures the ``importance" of each data point. This paper gives a novel sampling scheme from unlabeled data by defining the importance of each data point as the {\it{contribution ratio to effective dimension}}. \par
\subsection*{Main Contributions}
\begin{itemize}
    \item We propose a new importance labeling scheme called CRED (Contribution Ratios to Effective Dimension), which employs so-called \emph{contribution ratio} as the importance of each data point so that we can efficiently exploit information of input data. The contribution ratio measures how each data point contributes to the \emph{effective dimensionality} of RKHS which plays the essential role for characterizing the estimation performance of kernel ridge regression. 
    \item The generalization error of gradient descent on the labeled dataset selected by CRED is theoretically analysed in the settings of kernel ridge regression. It is shown that our algorithm achieves wider optimality than existing methods in general settings and significantly better generalization ability particularly under low label noise (i.e., near interpolation) settings.
    \item The algorithm and the theoretical results are extended to random features settings and the potential computational intractability of CRED from infinite dimensionality of RKHS is resolved.
\end{itemize}
The comparison of theoretical generalization errors between our proposed algorithms with the most relevant existing methods is summarised in Table \ref{tab: theoretical_comparison}. 

\begin{table*}[t]
    \centering
    \scalebox{0.9}{
    \begin{tabular}{c c c}\hline
        Method & Generalization Error & Additional Assumptions \\ \hline 
        (S)GD \citep{pillaud2018statistical}& $\left(\frac{C}{n}\right)^\frac{2r}{\mu}+ \left(\frac{\sigma^2\mathrm{Tr}(\Sigma^\frac{1}{\alpha})}{n}\right)^\frac{2r\alpha}{2r\alpha+1}$ & $\exists\mu \in [\frac{1}{\alpha}, 1]:\|\Sigma^{\frac{\mu}{2}-\frac{1}{2}}K_x\|_H^2 \leq C$ a.e. $x$ \\ 
        KTR$^3$ \citep{jun2019kernel} &  $\left(n^{-2r}+\left(\frac{\sigma^2}{n}\right)^\frac{2r}{2r+1}\right)\wedge \left(\frac{M^2\mathrm{Tr}(\Sigma^\frac{1}{\alpha})}{n}\right)^\frac{2r\alpha}{2r\alpha+1}$ & None \\ 
        SSSL \citep{ji2012simple}& $n^{-\frac{(\alpha-1)}{2}}$ & $r\geq0.5$, $\sigma^2 = 0$ \par
        sufficiently large $N$ \\
        \color{red}CRED-GD (this paper) & \color{red}$\left(\frac{\mathrm{Tr}(\Sigma^{\frac{1}{\alpha}})}{n}\right)^{2r\alpha} +
        \left(\frac{\sigma^2\mathrm{Tr}(\Sigma^{\frac{1}{\alpha}})}{n}\right)^{\frac{2r\alpha}{2r\alpha+1}}$ & \color{red}sufficiently large $N$ \\ \hline 
        RF-KRLS \citep{rudi2017generalization}& $n^{-2r} + \left(\frac{\sigma^2\mathrm{Tr}(\Sigma^\frac{1}{\alpha})}{n}\right)^\frac{2r\alpha}{2r\alpha+1}$ & $r \geq 0.5$, sufficiently large $m$ \\ 
        \color{red}RF-CRED-GD (this paper) & \color{red}$\left(\frac{\mathrm{Tr}(\Sigma^{\frac{1}{\alpha}})}{n}\right)^{2r\alpha} +
        \left(\frac{\sigma^2\mathrm{Tr}(\Sigma^{\frac{1}{\alpha}})}{n}\right)^{\frac{2r\alpha}{2r\alpha+1}}$ & \color{red}sufficiently large $m, N$ \\ \hline         
    \end{tabular}
    }
    \caption{Comparison of theoretical generalization errors between our proposed algorithms and most relevant existing methods (The bottom two methods use approximation by $m$ random features). $n$ is the number of labeled data, $\sigma^2$ is the variance of label noise, $M$ is the uniform upper bound of labels, $r \in [0, 1]$ represents the smoothness of the target function and $\alpha > 1$ captures the simplicity of the feature space. In column ``Additional Assumptions," $N$ means the number of unlabeled data. Please refer to Section \ref{sec: problem_settings_and_assumtions} for the detailed definitions of these parameters. Extra log factors $\mathrm{poly}(\mathrm{log}(n), \mathrm{log}(\delta^{-1}))$ are hided for simplicity, where $\delta$ is confidence parameter for high probability bounds.}
    \label{tab: theoretical_comparison}
\end{table*}

\subsection*{Related Work}
Here, we briefly overview the most relevant research areas and methods to our work. \par
{\bf{Supervised Learning.}} 
Supervised least squares regression in RKHS has been thoroughly studied \citep{yao2007early,caponnetto2007optimal,steinwart2009optimal,rosasco2015learning,dieuleveut2016nonparametric, rudi2017generalization, lin2017optimal, carratino2018learning,pillaud2018statistical, jun2019kernel}. \cite{caponnetto2007optimal, steinwart2009optimal} have shown the minimax optimal generalization ability of kernel ridge regression  under suitable assumptions. In \cite{yao2007early, rosasco2015learning},  gradient descent for kernel ridgeless regression has been considered and the effect of early stopping as implicit regularization has been theoretically justified. The analysis has been further improved with additional assumption about eigenvalues decay of the covariance operator of the feature space \citep{lin2017optimal}. Online stochastic gradient descent (SGD) has been studied in \citep{dieuleveut2016nonparametric} and the minimax optimal rate has been established when the true function is (nearly) attainable.  Recently the authors of \citep{pillaud2018statistical} have considered Multi-Pass SGD and shown its optimality without attainability of the true function under additional assumption about the capacity of the feature space in terms of infinity norm. Random features technique \citep{rahimi2008random} can be applicable to kernel regression and reduces the computational time. The generalization ability of kernel regression with random features has been studied in \cite{rudi2017generalization, carratino2018learning} and it has been shown that random features technique doesn't hurt the generalization ability when the number of random features is sufficiently large and the true function is attainable. More recently, in \citep{jun2019kernel}, low label noise cases have been particularly discussed and their proposed Kernel Truncated  Randomized Ridge Regression (KTR$^3$) achieves an improved rate when the label noise is low. However, these papers do not consider the utilization of the unlabeled data and hence the generalization ability may be sub-optimal in the importance labeling settings considered in this paper. \par
{\bf{Semi-Supervised Learning.}}  Semi-supervised learning has a close relation to importance labeling. In semi-supervised learning, we are given many unlabeled data and small number of labeled data. Typically the labeled data is uniformly selected from unlabeled data. Semi-supervised learning aims to get better generalization ability by the effective use of unlabeled examples typically under so-called cluster assumption \citep{balcan2005pac,rigollet2007generalization,ben2008does,wasserman2008statistical}. In contrast, the importance labeling scheme in this paper aims to get better generalization ability by the effective choice of labeled examples without the assumption. In \cite{ji2012simple}, a simple semi-supervised kernel regression algorithm called SSSR has been proposed and they have shown that the generalization ability surpasses the one of supervised learning when the true function is attainable and deterministic. Roughly speaking, the algorithm first computes eigen-system of covariance operator in the feature space using unlabeled data. Then, linear regression is executed on the principle eigen-functions as features. The theory of SSSR does not require the cluster assumption and is on the standard theoretical settings of kernel regression, but the generalization ability may be still sub-optimal. \par
{\bf{Active Learning}}. Active learning is also a close concept to importance labeling. In active learning, we are given learned model on small labeled data and then select new labeled data from unlabeled one by utilizing the information of the learned model. In some sense, active learning is a generalized concept of important labeling. However, in active learning, how to select the  initially labeled data is out-of-scope and typically assumed to be uniform selection. Enormous active learning strategies have been proposed \citep{brinker2003incorporating,dasgupta2005analysis,yu2006active,kapoor2007active,guo2008discriminative,wei2015submodularity,gal2017deep,sener2017active} (\citep{settles2009active} for extensive survey) and empirically studied their performances but their theoretical aspects are little known at least in our kernel regression setting. \par
{\bf{Importance Sampling.}} Importance sampling is a general technique to reduce the variance of estimations and typically used in Monte Carlo methods and stochastic optimization \citep{needell2014stochastic,zhao2015stochastic,alain2015variance,csiba2018importance,chen2019fast}. The behind idea is that if the realizations that potentially cause large variance are more frequently sampled, the variance of a bias-corrected estimator can be reduced. However, the definition of importance is strongly problem-dependent and to the best of our knowledge, any algorithms for importance labeling problem have not been proposed so far.

\section{Problem Settings and Assumptions}\label{sec: problem_settings_and_assumtions}
In this section, we provide the formal problem settings in this paper and theoretical assumptions for our analysis.
\subsection{Kernel Regression with Importance Labeling}
Let $Z_N = \{(x_j, y_j)\}_{j=1}^N$ be i.i.d. samples from some distribution $\rho_\mathcal{Z}$, where $z_j = (x_j, y_j) \in \mathcal{X}\times \mathcal{Y} \subset \mathbb{R}^d\times\mathbb{R}$, and $X_N = \{x_j\}_{j=1}^N$, $\bm y_N = \{y_j\}_{j=1}^N$. We denote $\rho_\mathcal{X}$ as the marginal distribution of $\mathcal{Z}$ on $\mathcal{X}$ and $\rho_{\mathcal{Y}|x}$ as the conditional distribution of $\mathcal{Y}$ with respect to $x \in \mathcal{X}$. We subsample $Z_n = \{(x_{j(i)}, y_{j(i)})\}_{i=1}^n$ ($n < N$) from $Z_N$ according to user-defined distribution $q$ on $Z_N$ and we denote $X_n = \{x_{j(i)}\}_{i=1}^n$, $\bm y_n = \{y_{j(i)}\}_{i=1}^n$. \par
The objective of this paper is to minimize the excess risk $\mathcal{E}(f) - \mathrm{inf}_{f' \in H} \mathcal{E}(f')$ only using the information of labeled observations $Z_n$, where $\mathcal{E}(w) = \int_\mathcal{Z} \frac{1}{2}(y - w(x))^2d\rho_\mathcal{Z}(x, y)$ and $H \subset L^2(\rho_\mathcal{X}) (\subset \mathbb{R}^\mathcal{X})$ is some 
Reproducing Kernel Hilbert Space (RKHS) with inner product $\langle \cdot, \cdot \rangle_H: H\times H \to \mathbb{R}$ and kernel $K(\cdot, \cdot): \mathcal{X}\times\mathcal{X} \to \mathbb{R}$. 
\paragraph{Notation} We denote by $\|\cdot\|_H$ the norm induced by $\langle \cdot, \cdot\rangle_H$ and $\|\cdot\|_2$ as the Euclidean norm. Let $\Sigma = S^*S: H \to H$ and $\mathcal{L}=SS^*: L^2(\rho_\mathcal{X}) \to L^2(\rho_\mathcal{X})$, where the operator $S$ is the natural embedding from $H$ to $L^2(\rho_\mathcal X)$ and $S^*$ is the adjoint operator of $S$. We define $T_\lambda$ as $T+\lambda I$ for operator $T$. For natural number $m$, We denote $\{1, \ldots, m\}$ by $[m]$. $K_x$ denotes the operator $K(x, \cdot) = K(\cdot, x): \mathcal{X}\to\mathbb{R} \in H$ for $x \in \mathcal{X}$. $K_x$ can be regard as a ``feature" of input $x$.

\subsection{Theoretical Assumptions}
We make the following assumptions for our theoretical analysis. These are fairly standard in the literature of statistical learning theory for kernel methods  \citep{steinwart2009optimal,dieuleveut2016nonparametric, lin2017optimal,pillaud2018statistical}. 
\begin{assumption}[Boundedness of feature]\label{assump: kernel_boundedness}
 For some $\kappa > 0$, $\mathrm{sup}_{x\in\mathrm{supp}(\rho_{\mathcal X})} \|K_x\|_H \leq \kappa$.
\end{assumption}

\begin{assumption}[Smoothness of true function]\label{assump: smoothness_of_true}
There exists $r \in (0, 1]$ such that $f_* = \mathcal L^r \phi$ for some $\phi \in L^2(\rho_{\mathcal X})$ with $\|\phi\|_{L^2(\rho_{\mathcal X})} \leq R$ ($R>0$). Here $f_*(\cdot) = \int_\mathcal{Y}yd\rho_{\mathcal{Y}|\cdot}$ that is the regression function (or true function).
\end{assumption}
Assumption \ref{assump: smoothness_of_true} quantifies the complexity of true function $f_*$ in terms of the eigen-system of $\mathcal L$. It is known that when $r \geq 1/2$, $\mathcal{L}^r(L^2(\rho_\mathcal X))$ becomes a subset of $H$ and particularly $r = 1/2$, it exactly matches to $H$. Thus, we have $f_* \in H$ whenever $r \geq 1/2$. However, when $r < 1/2$, generally $f_* \notin H$. As $r \to 0$, roughly $\mathcal{L}^r(L^2(\rho_\mathcal X)) \to L^2(\rho_\mathcal X)$. This means that $f_*$ can be more complex (or non-smooth) for smaller $r$.

\begin{assumption}[Polynomial decay of eigenvalues]\label{assump: eigen_decay}
There exists $\alpha > 1$ such that $\mathrm{Tr}(\Sigma^{1/\alpha}) < \infty$. 
\end{assumption}
Parameter $\alpha$ characterizes the complexity of feature space $H$. For larger $\alpha$, the feature space becomes more simple and particularly when $\alpha = \infty$, the feature space must have finite dimension. Note that even for feature spaces with finite dimensionality $d$, discussions of the case $\alpha < \infty$ are important because $\mathrm{Tr}(\Sigma^{1/\alpha})$ can be much smaller than $\mathrm{Tr}(\Sigma^{1/\infty}) = d$ for some $\alpha \in (1, \infty)$.

\begin{assumption}[Bounded variance and uniform bounededness of labels]\label{assump: noise}
 There exists $\sigma \geq 0$ and $M \geq 1$ such that $\mathbb{E}(y - f_*(x))^2 \leq \sigma^2$ and $|y|\leq M$ almost surely.
\end{assumption}
Generally label noise $\sigma > 0$, but we are particularly interested in the case $\sigma \to 0$. 

\section{Proposed Algorithm}
In this section, first the behind ideas are described and then formal descriptions of the proposed algorithm are given. \par
{\bf{Behind Ideas. }}Our proposed importance labeling scheme is based on the contribution ratios to {\it{effective dimension}} which plays the essential role for characterizing the estimation performance of kernel ridge regression \citep{zhang2005learning}. First recall the notion of effective dimension $\mathcal{N}_\infty(\lambda) = \mathbb{E}_x\|\Sigma_\lambda^{-1/2}K_x\|_H^2$, that is roughly the mean of the squared Mahalanobis distances of the features if $\mathbb{E}_x[K_x] = 0$. The essential intuition of our scheme is that labeling input $x$ that has a large contribution to effective dimension reduces the estimation variance . To realize this intuition, we construct an importance sampling distribution proportional to  $\|\Sigma_\lambda^{-1/2}K_x\|_H^2$ on unlabeled data samples. After sampling the data to be labeled, correcting the bias of the empirical risk caused by the importance labeling is needed. This situation is very similar to the one in the well-known {\it{importance sampling}} in the literature of classical Monte Carlo methods. \par
Next, for supporting the intuition and understanding how our sampling scheme works, we conduct simple synthetic experiments. We focus on a two dimensional feature space in $\mathbb{R}^2$. First we generated $100,000$ unlabeled samples $\{(x_1^{(i)}, x_2^{(i)})\}_{i=1}^{100000}$ according to $X_1 \sim N(0, 1)$ and $X_2 \sim N(0, 0.01)$ independently. For comparing our scheme with uniform labeling, we labeled $100$ data samples from unlabeled one using two sampling scheme independently. Figure \ref{fig: sampling_comp} shows the comparison of the labeled data by the two schemes. We can see that the data samples labeled by our proposed CRED covers a wider range of areas than uniform labeling. For making sure that CRED reduces the estimation variance, we conducted $1,000$ runs of least square regression on randomly labeled $3$ data samples using CRED and uniform labeling independently. We set true function $f_*$ to $f_*(x_1, x_2) = x_1 + x_2$ and added Gaussian noise with mean zero and variance $0.01$ for generating labels. Note that for each labeled sample we multiplied the inverse of the labeling probability of the sample to the correspondence loss and corrected the bias of the empirical risk caused by the importance labeling as in the standard importance sampling scheme. Figure \ref{fig: estimations_dev} shows the comparison of the deviation of the estimated regression coefficients. We can see that CRED in fact significantly reduces the estimation variance. 

\begin{figure}[t]
    \centering
    \hspace{-0.09\hsize}
    \includegraphics[width=7.9cm]{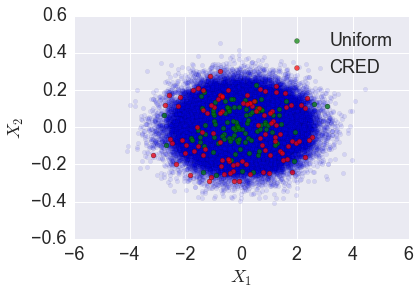}
    \caption{Comparision of the selected (labeled) data samples. (Blue) Unlabeled data ($100,000$ points). (Green) Labeled data selected by uniform sampling ($100$ points). (Red) Labeled data samples selected by CRED ($100$ points). }
    \label{fig: sampling_comp}
\end{figure}
\begin{figure}[t]
    \centering
    \includegraphics[width=7cm]{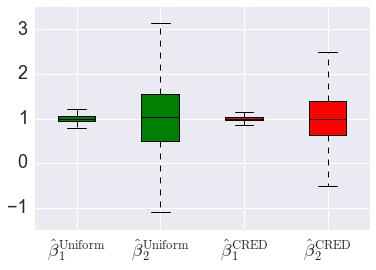}
    \caption{Comparison of the deviation of the estimated regression coefficients using $3$ labeled points selected by uniform sampling and CRED ($1,000$ independent trials). The true coefficients were $(\beta_1^*, \beta_2^*) = (1, 1)$. }
    \label{fig: estimations_dev}
\end{figure}

{\bf{Concrete Algorithm. }}Our proposed algorithm is illustrated in Algorithm \ref{alg: cred_il}. The algorithm consists of two blocks of importance labeling and optimization by gradient descent. \par
First we select a subset of the unlabeled data using a sampling distribution proportional to $\|\Sigma_\lambda^{-1/2}K_x\|_H^2$ on unlabeled data $x$, that can be regard as contribution ratio to effective dimension. For stability of sampling, we add the mean of the contribution ratios to it. Finally, since covariance operator $\Sigma$ is unknown, we replace it by empirical covariance operator $\Sigma_{N, \lambda}$ using $N$ unlabeled data. Line 1 in Algorithm \ref{alg: cred_il_rf} gives the formal description of this procedure. \par
Next, we run the standard gradient descent to minimize the empirical risk estimated by the labeled data, but each loss is weighed by the inverse labeling probability to guarantee the unbiasedness of the risk. Thus, the gradient of the bias corrected risk is used for updating the solution. Concretely, since gradient at $g$ with respect to given single observation $(x, y)$ is $(\langle K_x, g \rangle_H - y)K_x = (K_x \otimes K_x)g - yK_x$, if the sampling probability of $(x, y)$ from $N$ unlabeled data is $q$, we need to correct the bias of the sampling by multiplying a factor $1/(Nq)$ to the gradient. Then all the gradient with respect to labeled data is averaged. The formal description of this procedure is given in Line 5-6. Note that when the labeling distribution is uniform, i.e., $q = 1/N$, the algorithm matches to the standard gradient descent.\par 

\begin{remark}[Computational Tractability]
Gradient descent on RKHS can be efficiently executed even in infinite dimensional feature spaces thanks to kernel trick. However the computation of the contribution ratios to effective dimension is generally intractable due to the inapplicability of kernel trick \citep{scholkopf2002learning}. This computational problem can be avoided by introducing random features technique. For the details, see Section \ref{sec: random_features}.
\end{remark}

\begin{algorithm}[t]
\label{alg: cred_il}
\caption{CRED-GD($\eta$, $\lambda_q$, $T$)}
\begin{algorithmic}[1]
\STATE Set $q_j = \frac{\left\|\Sigma_{N, \lambda_q}^{-\frac{1}{2}}K_{x_j}\right\|_H^2+\frac{1}{N}\sum_{i=1}^N \left\|\Sigma_{N, \lambda_q}^{-\frac{1}{2}}K_{x_j}\right\|_H^2}{2\sum_{j=1}^N\left\|\Sigma_{N, \lambda_q}^{-\frac{1}{2}}K_{x_j}\right\|_H^2}$ for $j \in [N]$.
\STATE Sample $\{x_{j(i)}\}_{i=1}^n$ independently according to $q$ and get their labels $\{y_{j(i)}\}_{i=1}^n$.
\STATE Set $g_0 = 0$.
\FOR {$t=1$ to $T$}
\STATE $A = \frac{1}{n}\sum_{i=1}^n \frac{1}{Nq_{j(i)}}(K_{x_{j(i)}} \otimes K_{x_{j(i)}})$, $b=\frac{1}{n}\sum_{i=1}^n \frac{1}{Nq_{j(i)}} y_{x_{j(i)}}K_{x_{j(i)}}$.
\STATE $g_t = g_{t-1} - \eta\left(A g_{t-1} - b\right)$.
\ENDFOR
\RETURN $g_T$. 
\end{algorithmic}
\end{algorithm}

\section{Generalization Error Analysis}\label{sec: generalization} 
Here, we give the main theoretical results of CRED-GD (Algorithm \ref{alg: cred_il}). The proofs are found in Section \ref{app_sec: main_results} of the supplementary material. We use $\widetilde{O}$ and $\widetilde \Omega$ notation to hide extra $\mathrm{poly}(\mathrm{log}(n), \mathrm{log}(\delta^{-1}))$ factors for simplicity, where $\delta$ is a confidence parameter for high probability bounds. \par
Our analysis starts from bias-variance decomposition $\left\|Sg_t - f_*\right\|_{L^2(\rho_{\mathcal X})}^2 \leq 2 \left\|Sf_t - f_*\right\|_{L^2(\rho_{\mathcal X})}^2 + 2\left\|S(g_t - f_t)\right\|_{L^2(\rho_{\mathcal X})}^2$,
where $\{f_t\}_{t=1}^\infty$ is the ideal GD path on excess risk, i.e., $f_t = f_{t-1} - \eta(\Sigma f_{t-1} - S^*f_*) = f_{t-1} - \eta (\mathbb{E}_x[K_x\otimes K_x] - \mathbb{E}_{x, y}[yK_x])$ with $f_0 = 0$. The first term is called as bias and the second term is called as variance. The bias can be bounded by the following Proposition:
\begin{proposition}[{\bf{Bias bound}}, simplified version of Lemma \ref{lem: f_boundedness}]\label{main_lem: f_boundedness}
Suppose that Assumptions \ref{assump: kernel_boundedness} and \ref{assump: smoothness_of_true} hold.
Let $\eta = O(1/\kappa^2)$ be sufficiently small. Then, for any $t \in \mathbb{N}$,
\begin{align*}
    \|Sf_t - f_*\|_{L^2(\rho_\mathcal X)}^2 = O\left(R^2(\eta t)^{-2r}\right).
\end{align*}
\end{proposition}
Lemma \ref{main_lem: f_boundedness} shows that the bias converges to $0$ as $t \to \infty$. Moreover, the convergence speed is controlled by the smoothness of the true function.  
\begin{definition}
We define $\mathcal N_\infty (\lambda) = \mathbb{E}_x\|\Sigma_\lambda^{-1/2}K_x\|_H^2$ and $\mathcal F_\infty (\lambda) = \mathrm{sup}_{x\in\rho(\mathcal X)} \|\Sigma_\lambda^{-1/2}K_x\|_H^2$.
\end{definition}
These quantities play the essential roles for characterizing the estimation performance. We can bound these quantities as follows:
\begin{lemma}\label{main_lem: F_N_bound}
Suppose that Assumption \ref{assump: kernel_boundedness} holds. For any $\lambda > 0$, $\mathcal F_\infty(\lambda) \leq \kappa^2\lambda^{-1}$.  
Additionally, under Assumption \ref{assump: eigen_decay}, for any $\lambda > 0$, $\mathcal N_\infty(\lambda) \leq \mathrm{Tr}(\Sigma^{1/\alpha})\lambda^{-1/\alpha}$.
\end{lemma}
Since $\alpha > 1$, $\mathcal N_\infty(\lambda)$ has a much tighter bound than $\mathcal F_\infty(\lambda)$ for small $\lambda$. \par
Now, we bound the second term, that is called as variance, using th following proposition:
\begin{proposition}[{\bf{Variance bound}}, simplified version of Proposition \ref{prop: cred_il}]\label{main_prop: cred_il}
Suppose that $\eta = O(1/\kappa^2)$ be sufficiently small. Let $t \in \mathbb{N}$, $\lambda =1/(\eta t) \geq \lambda_q = \Omega((\mathrm{Tr}(\Sigma^{1/\alpha})/n)^\alpha)$, $\delta \in (0, 1)$ and $n \geq \widetilde \Omega(1+\mathrm{Tr}(\Sigma^{1/\alpha})\lambda_q^{-1/\alpha})$ and $N \geq \widetilde \Omega(1+\kappa^2\lambda_q^{-1})$. Then there exits event $A$ with $P(A) \geq 1 - \delta$ such that
\begin{align*}
    &\mathbb{E}\left[\|S(g_t - f_t)\|_{L^2(\rho_{\mathcal X})}^2 \mid A\right] \\
    =&\ \widetilde O\left( \frac{(\sigma^2+R^2\lambda^{2r})\mathcal N_\infty(\lambda_q)}{n} + \lambda^{2r} + r_N\right),
\end{align*}
where $r_N = \mathcal F_\infty(\lambda_q)(\sigma^2+R^2\lambda^{2r}+(M^2+\kappa^{4r-2}R^2+R^2\lambda^{-1+2r}/N))/(nN) \to 0$ as $N \to \infty$.
\end{proposition}
Proposition \ref{main_prop: cred_il} shows that the variance diverges to $\infty$ as $t \to \infty$ (because $\mathcal N_\infty(\lambda) = O(\lambda^{-1/\alpha}) \to \infty$ as $t \to \infty$), but is scale to $1/n$. Thus, for moderate $t$, the variance can still be small. 

\begin{remark}
Proposition \ref{main_prop: cred_il} is the main novelty of our analysis. In \citep{pillaud2018statistical}, the variance bound of the standard GD is roughly $(\sigma^2\mathcal N_\infty(\lambda)+ \lambda^{2r}\mathcal F_\infty(\lambda))/n$ in our settings. In contrast, our bound is roughly $(\sigma^2+\lambda^{2r})\mathcal N_\infty(\lambda)/n$ for $\lambda \approx  \lambda_q$ and sufficiently large $N$ (note that $\lambda^{2r}$ can be ignored because it never dominates the bias term (see Proposition \ref{main_lem: f_boundedness})). Since $\mathcal N_\infty(\lambda) \leq \mathcal F_\infty(\lambda)$ always holds, CRED-GD improves the variance bound of the standard GD when $\sigma^2$ is small. Later, we discuss the case $\mathcal N_\infty(\lambda) \ll \mathcal F_\infty(\lambda)$ (see Lemma \ref{main_lem: F_N_bound} and Section \ref{sec: sufficient_condition}).
\end{remark}

\begin{remark}
In \cite{pillaud2018statistical}, under Assumption \ref{assump: kernel_boundedness} and additional assumption $\mathrm{sup}_{\mathrm{supp}(\rho_{\mathcal X})}\|\Sigma^{\mu/2-1/2}K_x\|_H^2 = O(\kappa_\mu^2 R^{2\mu})$ for some $\kappa_\mu > 0$ and $\mu \in [0, 1]$, the authors have shown that $\mathcal F_\infty(\lambda) = O(\kappa_\mu^2R^{2\mu}\lambda^{-\mu})$ (Lemma 13 in \citep{pillaud2018statistical}), which is a better bound than ours in Lemma \ref{main_lem: F_N_bound} when $\mu < 1$. However, in the worst case $\mu = 1$ their bound matches to ours in Lemma \ref{main_lem: F_N_bound}. For an example of this case, see Section \ref{sec: sufficient_condition}.
\end{remark}

For balancing the bias and variance term, we introduce a notion of the {\it{ optimal number of iterations}}:  
\begin{definition}[Optimal number of iterations]\label{def: opt_num_iter}
Optimal number of iterations for CRED-GD $t_\eta^*$ is defined by
$t_\eta^* = \left\lceil 1/(\eta \lambda_*)\right\rceil$,
where $\lambda_*$ is defined as
\begin{align*}
    \lambda_* = \widetilde{O}\left(
    \left(\frac{\sigma^2\mathrm{Tr}(\Sigma^{\frac{1}{\alpha}})}{n}\right)^{\frac{\alpha}{2r\alpha+1}}
    + 
    \left(\frac{R^2\mathrm{Tr}(\Sigma^{\frac{1}{\alpha}})}{n}\right)^\alpha 
    +
    \lambda_N
    \right),
\end{align*}
where $\lambda_N = \{\kappa^2\sigma^2+ \kappa^2M^2/N+\kappa^{4r}R^2/N)/(nN)\}^{1/(1+2r)} + \kappa^2R^2/(nN) + \kappa R/(\sqrt{n}N) \to 0$ as $N \to \infty$.
\end{definition}

Lemma \ref{main_lem: F_N_bound} and Proposition \ref{main_prop: cred_il} with $\lambda_q = \lambda_*$ yields the following main theorem:
\begin{theorem}[Generalization Error of CRED-GD]\label{main_thm: cred_il}
Suppose that Assumptions \ref{assump: kernel_boundedness}, \ref{assump: smoothness_of_true}, \ref{assump: eigen_decay} and \ref{assump: noise} hold. Let $\eta = \Theta(1/\kappa^2)$ be sufficiently small and $\delta \in (0, 1)$. Then setting $\lambda_q = \lambda_*$, $T = \widetilde \Theta (t_\eta^*)$, there exists event $A$ with $P(A) \geq 1 - \delta$ such that CRED-GD satisfies $\mathbb{E}\left[\left\|Sg_T - f_*\right\|_{L^2(\rho_{\mathcal X})}^2 \mid A \right] 
    = \widetilde{O}(\lambda_*^{2r})$,
where $\lambda_*$ is defined in Definition \ref{def: opt_num_iter}.
\end{theorem}
From Theorem \ref{main_thm: cred_il}, we obtain the following observations:
{\bf{Wider Optimality.}} When $\sigma^2 = \Theta(1)$, the generalization error of CRED-GD with  sufficiently many unlabeled data becomes the optimal rate $n^{-2r\alpha/(2r\alpha+1)}$. The same rate is also achieved by supervised GD or SGD but under restrictive condition $r > (\alpha-1)/(2\alpha)$ in our theoretical settings \citep{dieuleveut2016nonparametric,pillaud2018statistical}, which is not necessary for CRED-GD. \par
{\bf{Low Noise Acceleration.}} When $\sigma^2 \to 0$, the rate  of CRED-GD with sufficiently many unlabeled data becomes $n^{-2r\alpha}$. In contrast, supervised GD or SGD only achieves $O(n^{-2r})$ in our theoretical settings when $\sigma^2 \to 0$, and thus CRED-GD significantly improves the generalization ability of supervised methods. Semi-supervised method SSSL \citep{ji2012simple} only achieves $n^{-(\alpha-1)/2}$ when $\sigma^2=0$ and $r = 1/2$, which is worse than ours. \par
\begin{remark}[Equivalence to Kernel Ridge Regression with Importance Labeling]
Using very similar arguments of our analysis, it can be shown that analytical kernel ridge regression solution $(\Sigma_{n, \lambda_*}^{(q)})^{-1}(S_n^{(q)})^*\bm y_n$ also achieves the generalization error bound in Theorem \ref{main_thm: cred_il} (see Section \ref{app_sec: equivalence} of supplementary material). When $\lambda_*$ is extremely small, the analytical solution is computationally cheap than gradient descent and sometimes useful.
\end{remark}

\section{Sufficient Condition for $\mathcal N_\infty(\lambda) \ll \mathcal F_\infty(\lambda)$}\label{sec: sufficient_condition}
In this section, we give a sufficient condition for $\mathcal N_\infty(\lambda) \ll \mathcal F_\infty(\lambda)$ and its simple example. The proofs are found in Section \ref{app_sec: sufficient_condition} of the supplementary material.
\begin{proposition}\label{main_prop: F_lower_bound}
Let $\{(\lambda_i, \phi_i)\}_{i=1}^d$ ($d \in \mathbb{N}\cup\{\infty\}$) be the eigen-system of $\Sigma$ in $L^2(\rho_{\mathcal X})$, where $\lambda_1 \geq \lambda_2 \geq \ldots > 0$. Assume that $\lambda_i = \Theta(i^{-\alpha})$ and  $\|\phi_i\|_{L^\infty(\rho_{\mathcal X})} = \Omega(i^{p/2})$ for any $i$ for some $\alpha = 1 + \Omega(1)$ and $p \geq 1$. Moreover if $d = \infty$, we additionally assume $\|\phi_i\|_{L^\infty(\rho_{\mathcal X})}^2 = O(i^{\alpha-1-\varepsilon})$ for any $i$ for some $\varepsilon > 0$. Then Assumption \ref{assump: kernel_boundedness} is satisfied and for any $\lambda \in (0, 1)$, $\mathcal F_\infty(\lambda) = \Omega\left(\lambda^{-\frac{p}{\alpha}}\wedge d^p\right)$.
\end{proposition}
\begin{example}
Let $\mathcal X = [-1, 1]^d$ and $\rho_{\mathcal X} = TN(0, \sigma_1^2)\otimes\cdots\otimes TN(x_d, \sigma_d^2)$, that is the product measure of truncated normal distributions with mean $0$ and scale parameter $\sigma_i^2$, i.e., independent normal distributions with mean $0$ and variance $\sigma_i^2$ conditioned on $[-1, 1]$. Let $\sigma_1^2 \geq \cdots \geq \sigma_d^2$. 
We denote $\widetilde \sigma_i$ as the variance of $TN(0, \sigma_i^2)$ for $i \in [d]$. Note that for sufficiently small $\sigma_1^2 = \Theta(1)$, we have $\widetilde \sigma_i^2 = \Theta(\sigma_i^2)$ for any $i \in [d]$. Then we particularly consider linear kernel $K$ and thus $H = \mathcal X$. Since the covariance matrix is $\Sigma = \mathrm{diag}(\widetilde \sigma_1^2, \ldots, \widetilde \sigma_d^2)$, the eigen-system of $\Sigma$ in $L^2([-1, 1]^d)$ is $\{(\widetilde \sigma_i^2, e_i/\sqrt{\widetilde \sigma_i^2})\}_{i=1}^d$, where $e_i(x) = x_i$ for $x \in [-1, 1]^d$. Suppose that the polynomial decay of $\{\sigma_i^2\}_{i=1}^d$ holds: $\sigma_i^2 = \Theta(\widetilde \sigma_i^2) = \Theta(i^{-\alpha})$. Then from Lemma \ref{main_lem: F_N_bound}, $\mathcal N_\infty(\lambda) = O(\mathrm{log}(d)\lambda^{-1/\alpha}\wedge d)$.
On the other hand, from Proposition \ref{main_prop: F_lower_bound} with $p \leftarrow \alpha$, we have $\mathcal F_\infty(\lambda) = \Omega(\lambda^{-1}\wedge d^{\alpha})$.
\end{example}

\section{Extension to Random Features Settings}\label{sec: random_features}
In this section, we discuss the application of {\it{random features}} technique to Algorithm \ref{alg: cred_il} for computational tractability. Then we theoretically analyse the generalization error of the algorithm. The proofs are given in Section \ref{app_sec: random_features} of the supplementary material. \par
Suppose that kernel $K$ has an integral representation $K(x, x') = \mathbb{E}_{\omega \sim \pi}[\psi(x, \omega)\psi(x', \omega
)]$ for $x, x' \in \mathcal X$ for some $\psi$. Random features $\phi_{m, x}\in \mathbb{R}^m$ is defined by $m^{-1/2}(\psi(x, \omega_1), \ldots, \psi(x, \omega_m))$, where $\omega_1, \ldots, \omega_m \sim \pi$ independently, is used for an approximation of $K(x, x')$ by $\langle \phi_{m, x}, \phi_{m, x'}\rangle$. Here, the number of random features $m \in \mathbb{N}$ is a user-defined parameter and characterizes the goodness of the approximation. More details and concrete examples of random features are found in \cite{rudi2017generalization}.
{\bf{Algorithm. }}The random features version of CRED-GD is illustrated in Algorithm \ref{alg: cred_il_rf}. The difference from Algorithm \ref{alg: cred_il} is only the replacement of $K_x$ to random features $\phi_{M,x}$. Note that we can properly compute important labeling distribution $q$ using standard SVD solvers thanks to the finite dimensionality of the random features.  \par

\begin{algorithm}[t]
\label{alg: cred_il_rf}
\caption{RF-CRED-GD($\eta$, $\lambda_q$, $m$, $T$)}
\begin{algorithmic}[1]
\STATE Sample $\omega_1, \ldots, \omega_m \sim \pi$ independently.
\STATE Set $\phi_m(x_j) = m^{-\frac{1}{2}}(\phi(x_j, \omega_1), \ldots, \phi(x_j, \omega_m))$ for $j \in [N]$.
\STATE Set $q_j = \frac{\left\|\hat \Sigma_{N, \lambda_q}^{-\frac{1}{2}}\phi_{m, x_j}\right\|_2^2+\frac{1}{N}\sum_{i=1}^N \left\|\hat \Sigma_{N, \lambda_q}^{-\frac{1}{2}}\phi_{m, x_j}\right\|_2^2}{2\sum_{j=1}^N\left\|\hat \Sigma_{N, \lambda_q}^{-\frac{1}{2}}\phi_{m, x_j}\right\|_2^2}$ for $j \in [N]$, where $\hat \Sigma_{N, \lambda_q} = \frac{1}{N}\sum_{j=1}^N \phi_{m, x_j} \phi_{m, x_j}^\top + \lambda_q I$.
\STATE Sample $\{x_{j(i)}\}_{i=1}^n$ independently according to $q$ and get their labels $\{y_{j(i)}\}_{i=1}^n$.
\STATE Set $\hat g_0 = 0$.
\FOR {$t=1$ to $T$}
\STATE $A = \frac{1}{n}\sum_{i=1}^n \frac{1}{Nq_{j(i)}}(\phi_{m, x_{j(i)}} \phi_{m, x_{j(i)}}^\top)$, $b = \frac{1}{n}\sum_{i=1}^n \frac{1}{Nq_{j(i)}} y_{x_{j(i)}}\phi_{m, x_{j(i)}}$.
\STATE $\hat g_t = \hat g_{t-1} - \eta\left(A\hat g_{t-1} - b \right)$.
\ENDFOR
\RETURN $\hat g_T$. 
\end{algorithmic}
\end{algorithm}

We need the following additional assumption abound the boundedness of the random features for theoretical analysis:
\begin{assumption}\label{assump: rf_boundedness}
$\mathrm{sup}_{x \in \mathrm{supp}(\rho_{\mathcal X}), \omega\in\mathrm{supp}(\pi)}|\psi(x, \omega)| \leq \kappa$ for some $\kappa > 0$. \par
\end{assumption}
For example, random features of Gaussian kernel satisfies this assumption \citep{rudi2017generalization}. \par
We define $\hat S: \mathbb{R}^m \to L^2(\rho_{\mathcal X})$ by $(\hat S f)(x) = \langle \phi_{m, x}, f\rangle$ and $\hat S^*$ by the adjoint of $\hat S$. Then we denote $\hat \Sigma = \hat S^* \hat S$ and $\hat L = \hat S \hat S^*$. \par
{\bf{Generalization Error Analysis.}} We consider  generalization error $\|\hat S \hat g_{t} - f_*\|_{L^2(\rho_{\mathcal X})}^2$. We decompose the generalization error to bias and variance $
    \|\hat S \hat g_t - f_*\|_{L^2(\rho_{\mathcal X})}^2 \leq  2\|\hat S \hat f_t - f_*\|_{L^2(\rho_{\mathcal X})}^2 
    + 2\|\hat S \hat g_t - \hat S\hat f_t\|_{L^2(\rho_{\mathcal X})}^2$,
where where $\{\hat f_t\}_{t=1}^\infty$ is the ideal path of GD with RF on excess risk, i.e., $\hat f_t = \hat f_{t-1} - \eta(\hat \Sigma f_{t-1} - \hat S^*f_*) = \hat f_{t-1} - \eta (\mathbb{E}_x[\phi_{m, x}\phi_{m, x}^\top] - \mathbb{E}_{x, y}[y\phi_{m, x}])$ with $\hat f_0 = 0$. The bias term can be bounded similar to Proposition \ref{main_lem: f_boundedness}:
\begin{proposition}[Bias bound for RF setting, simplified version of Lemma \ref{lem: rf_f_boundedness}]\label{main_lem: rf_f_boundedness}
Suppose that Assumptions \ref{assump: smoothness_of_true} and \ref{assump: rf_boundedness} hold.
Let $\eta = O(1/\kappa^2)$ be sufficiently small and $t \in \mathbb{N}$ such that $m = \widetilde \Omega(1+\kappa^2 \eta t)$. Then for any $\delta > 0$, with probability at least $1 - \delta$,   
\begin{align*}
    \left\|\hat S \hat f_t - f_*\right\|_{L^2(\rho_\mathcal X)}^2 = O\left(R^2(\eta t)^{-2r}\right).
\end{align*}
\end{proposition}
\begin{remark}
Compared to Lemma \ref{main_lem: f_boundedness}, additional condition $m = \widetilde \Omega(1+\kappa^2 \eta t)$ is assumed. This implies that to make bias small, appropriately large number of random features $m$ is required.
\end{remark}
The variance conditioned on random features $\{\omega_k\}_{k=1}^m$ can be bounded in a perfectly similar manner to the proof of Proposition \ref{main_prop: cred_il} with replacing $\mathcal N_\infty(\lambda_q)$ and $\mathcal F_\infty(\lambda)$ by random features approximations $\hat {\mathcal N}_\infty(\lambda_q)$ and  $\hat {\mathcal F}_\infty(\lambda)$ respectively. $\hat {\mathcal F}_\infty(\lambda)$ has a trivial bound $O(\lambda^{-1})$. The key lemma for bounding $\hat {\mathcal N}_\infty(\lambda_q)$ is the following:
\begin{lemma}[Proposition 10 in \cite{rudi2017generalization}] \label{main_lem: rf_N_diff_bound} Suppose that Assumption \ref{assump: rf_boundedness} holds.
We denote 
$\hat{\mathcal N}_\infty(\lambda) = \mathbb{E}_x\|\hat \Sigma_\lambda^{-1/2} \phi_{m, x}\|_2^2$ for $\lambda > 0$.
For any $\delta \in (0, 1)$ and sufficiently small $\lambda = O(1)$, if $m = \widetilde \Omega (1 + \kappa^2\lambda^{-1})$, with probability at least $1-\delta$ it holds that $\hat{\mathcal N}_\infty(\lambda) \leq 1.55\mathcal N_\infty(\lambda)$.
\end{lemma}
Combining the bias and variance bounds with Lemma \ref{main_lem: rf_N_diff_bound} yields the following theorem:
\begin{theorem}[Generalization error of CRED-GD with RF, simplified version of Theorem \ref{thm: rf_cred_il}]\label{main_thm: rf_cred_il}
Suppose that Assumptions \ref{assump: smoothness_of_true}, \ref{assump: eigen_decay}, \ref{assump: noise} and \ref{assump: rf_boundedness} hold. Let $\eta = \Theta(1/\kappa^2)$ be sufficiently small, $\lambda_q = \lambda_*$ and $T = \widetilde \Theta (t_\eta^*)$. For any $\delta \in (0, 1)$, if $m \geq \widetilde O(1 + \kappa^2\lambda_*^{-1})$, there exists event $A$ with $P(A) \geq 1 - \delta$ such that RF-CRED-GD has the same generalization error bounds as CRED-GD in Theorem \ref{main_thm: cred_il}.
\end{theorem}
Theorem \ref{main_thm: rf_cred_il} ensures that Algorithm \ref{alg: cred_il_rf} achieves still the same generalization ability as Algorithm \ref{alg: cred_il} when the number of random features $m$ is sufficiently large. 

\begin{figure*}[t]
\begin{subfigmatrix}{4}
\subfigure[LR on MNIST]{\includegraphics[width=4.2cm]{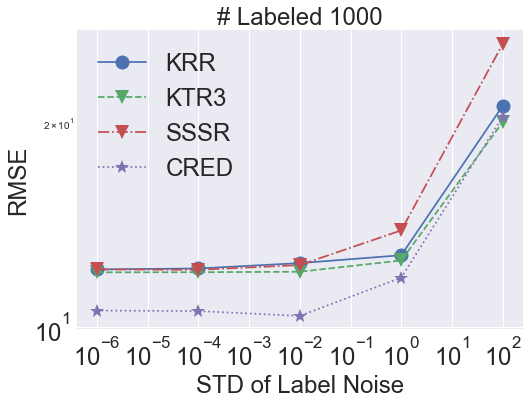}}
\subfigure[LR on Fashion ]{\includegraphics[width=4.2cm]{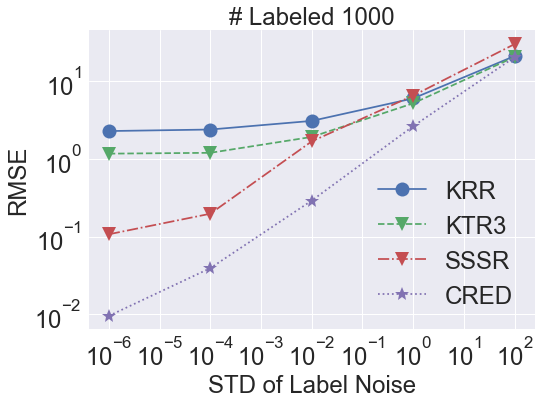}}
\subfigure[NLR on MNIST]{\includegraphics[width=4.2cm]{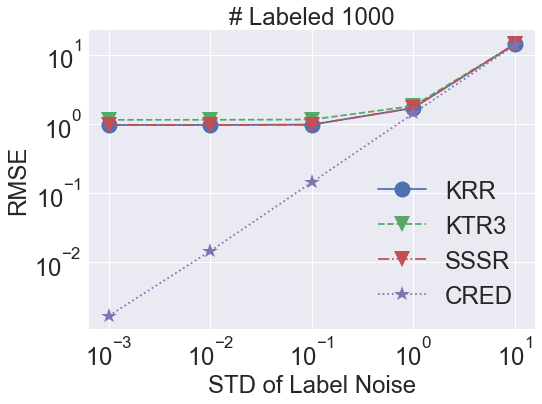}}
\subfigure[NLR on Fashion 
]{\includegraphics[width=4cm]{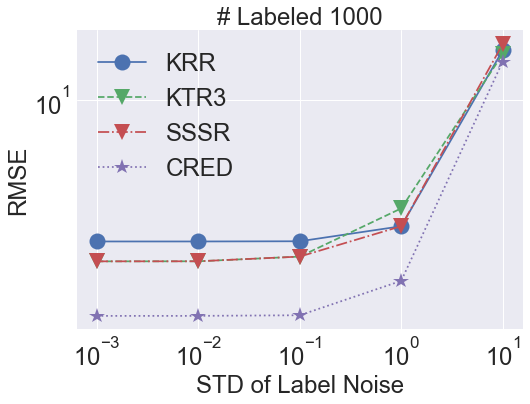}}
\end{subfigmatrix}
\begin{subfigmatrix}{4}
\subfigure[LR on MNIST]{\includegraphics[width=4.2cm]{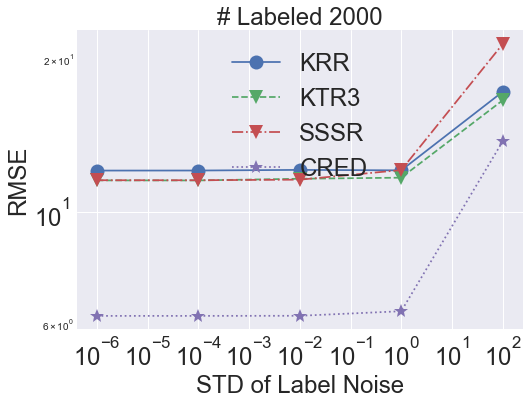}}
\subfigure[LR on Fashion ]{\includegraphics[width=4.2cm]{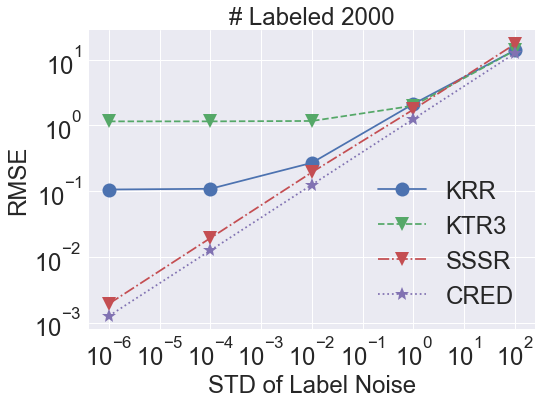}}
\subfigure[NLR on MNIST]{\includegraphics[width=4.2cm]{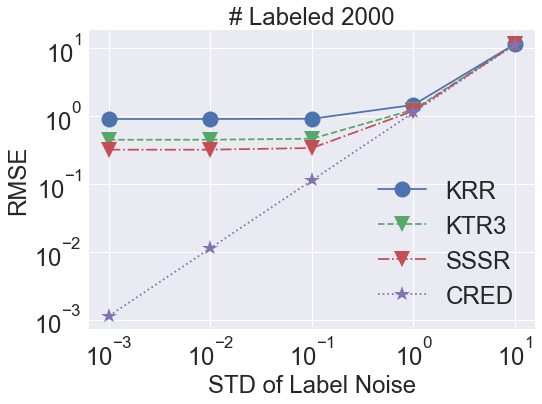}}
\subfigure[NLR on Fashion 
]{\includegraphics[width=4.2cm]{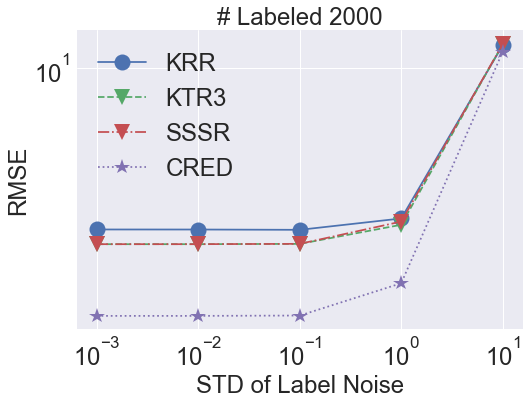}}
\end{subfigmatrix}
\begin{subfigmatrix}{4}
\subfigure[LR on MNIST]{\includegraphics[width=4.2cm]{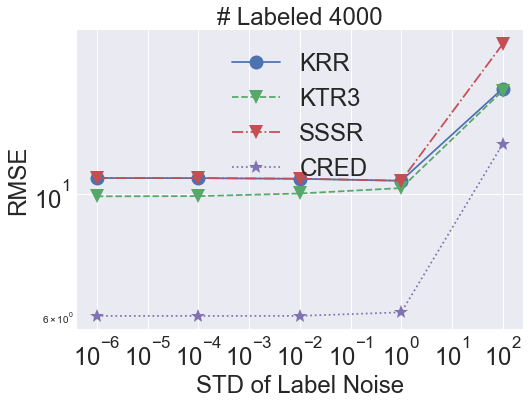}}
\subfigure[LR on Fashion ]{\includegraphics[width=4.2cm]{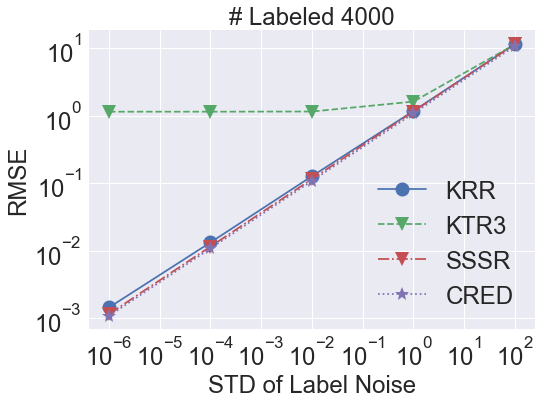}}
\subfigure[NLR on MNIST]{\includegraphics[width=4.2cm]{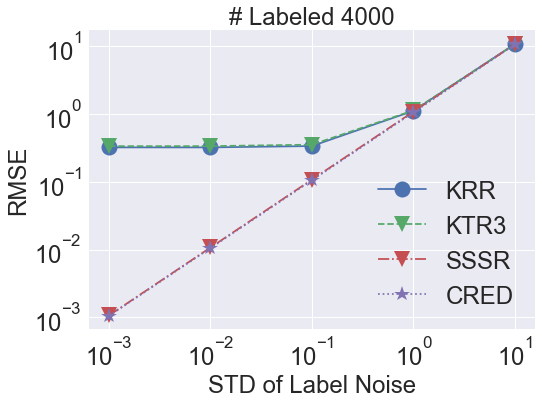}}
\subfigure[NLR on Fashion 
]{\includegraphics[width=4.2cm]{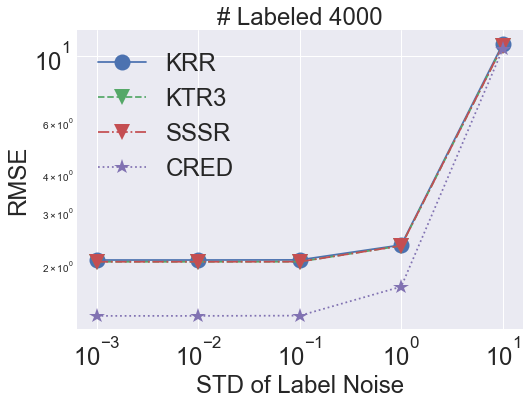}}
\end{subfigmatrix}
\caption{Comparisons of test RMSE of our method with existing methods on linear and nonlinear regression tasks. The first column depicts the results of linear regression task on MNIST. The second column does the ones of linear regression on Fashion MNIST. The third column does the ones of nonlinear regression on MNIST. The last column does the ones of nonlinear regression on Fashion MNIST. From top to bottom, the number of labeled data increases from $1000$ to $4000$.}
\label{fig: comparison}
\end{figure*}
\section{Numerical Experiments}
In this section, numerical results are provided to empirically verify our theoretical findings. \par
{\bf{Experimental Settings. }}In our experiments, the input data of public datasets MNIST and Fashion MNIST \citep{xiao2017fashion} were used. First we randomly split each dataset into train ($60,000$) and test ($10,000$) and normalized input data by dividing $255$. We conducted both linear regression (LR) and nonlinear regression (NLR) tasks. For linear tasks, we used the original inputs with bias as features. For nonlinear tasks, we used a randomly initialized three hidden layered fully connected ReLU network with width $500$ without output layer as features. Here, the random weights were from i.i.d. standard normal distributions. Then we randomly generated true linear function on the feature spaces, whose regression coefficients were defined by $\sum_i a_i e_i/\sqrt{\lambda_i} \in \mathbb{R}^{500}$, where $a_i \stackrel{\mathrm{i.i.d}}{\sim} N(0, 1)$ and $\{(e_i, \lambda_i)\}$ was the eigen-system of the covariance matrix in the correspondence feature space. Finally, we generated noised labels based on them, where the noises were from i.i.d. normal distributions with mean $0$ and variance $\sigma^2 \in \{10^{-6}, 10^{-4}, 10^{-2}, 1, 10^2\}$. 
We compared our proposed method\footnotemark with KRR (Kernel Ridge Regression), KTR$^3$ \citep{jun2019kernel} and SSSR \citep{ji2012simple}. 
\footnotetext{As we mentioned before, we used  very small synthetic label noise in some experiments and then the convergence speed of gradient descent was sometimes quite slow. Hence we decided that optimization methods were replaced with analytical methods. As we pointed out in the end of Section \ref{sec: generalization}, the same generalization error bound is guaranteed for the analytical solution.}
The hyper-parameters were fairly and reasonably determined.\footnotemark \footnotetext{CRED has hyper-parameter $\lambda_q$ and selecting best one requires additional labeling. In our experiments, we recorded the best test error by trying $\lambda_q$ in $\{10^i \mid i \in \{-12, \ldots, -3\}\}$. This potentially violates the fair comparison with the other methods because CRED implicitly uses ten patterns of labeled data. Hence we decided that the other methods were ran ten times with independent uniform labeling and then the best test error was recorded as one experimental trial. }
The train data was used as unlabeled data and the labeled data was selected from it. The number of labeled data was ranged in $\{1000, 2000, 4000\}$. We independently ran each experiment five times and recorded the median of test RMSE on each setting. \par

\setlength{\textfloatsep}{10pt}
{\bf{Results}}\ \ Figure \ref{fig: comparison} shows the comparisons of test RMSE of our proposed method with previous methods. From these results, we make the following observations:
\begin{itemize}
    \item When the label noise $\sigma^2$ was large, all the algorithms have similar performances. 
    \item When the label noise $\sigma^2$ was small, CRED significantly outperformed the other methods overall. SSSR was always comparable to or better than KRR and KTR3, but sometimes significantly worse than CRED. 
\end{itemize}
These observations can be well-explained by the theoretical results that show our proposed CRED achieves much better generalization ability than the other methods when $\sigma^2 \to 0$ as described in Table \ref{tab: theoretical_comparison}.
\section*{Conclusion and Future Work}
In this paper, we proposed a new importance labeling scheme called CRED, which employs the contribution ratio to the effective dimension of the feature space as the importance of each data point. The generalization error of GD with CRED was theoretically analysed and much better bound than previous methods was derived when label noise is small. Further, the algorithm and analysis were extended to random features settings and computational intractability of CRED was resolved. Finally, we provided numerical comparisons with existing methods. The numerical results showed empirical superiority to the other methods and verified our theoretical findings. \par
One direction of future work would be an application of our importance labeling idea to deep learning. Since the feature space of a deep neural network is updated in training time, our importance labeling scheme can be naturally extended to active learning settings. The theoretical and empirical study of the application to active learning of deep neural networks is a promising future work.

\section*{Acknowledgement}
TS was partially supported by JSPS KAKENHI (18K19793, 18H03201, and 20H00576), Japan
DigitalDesign, and JST CREST.

\bibliography{references}

\begin{thebibliography}{36}
\providecommand{\natexlab}[1]{#1}
\providecommand{\url}[1]{\texttt{#1}}
\expandafter\ifx\csname urlstyle\endcsname\relax
  \providecommand{\doi}[1]{doi: #1}\else
  \providecommand{\doi}{doi: \begingroup \urlstyle{rm}\Url}\fi

\bibitem[Alain et~al.(2015)Alain, Lamb, Sankar, Courville, and
  Bengio]{alain2015variance}
G.~Alain, A.~Lamb, C.~Sankar, A.~Courville, and Y.~Bengio.
\newblock Variance reduction in sgd by distributed importance sampling.
\newblock \emph{arXiv preprint arXiv:1511.06481}, 2015.

\bibitem[Balcan and Blum(2005)]{balcan2005pac}
M.-F. Balcan and A.~Blum.
\newblock A pac-style model for learning from labeled and unlabeled data.
\newblock In \emph{International Conference on Computational Learning Theory},
  pages 111--126. Springer, 2005.

\bibitem[Ben-David et~al.(2008)Ben-David, Lu, and P{\'a}l]{ben2008does}
S.~Ben-David, T.~Lu, and D.~P{\'a}l.
\newblock Does unlabeled data provably help? worst-case analysis of the sample
  complexity of semi-supervised learning.
\newblock In \emph{COLT}, pages 33--44, 2008.

\bibitem[Brinker(2003)]{brinker2003incorporating}
K.~Brinker.
\newblock Incorporating diversity in active learning with support vector
  machines.
\newblock In \emph{Proceedings of the 20th international conference on machine
  learning (ICML-03)}, pages 59--66, 2003.

\bibitem[Caponnetto and De~Vito(2007)]{caponnetto2007optimal}
A.~Caponnetto and E.~De~Vito.
\newblock Optimal rates for the regularized least-squares algorithm.
\newblock \emph{Foundations of Computational Mathematics}, 7\penalty0
  (3):\penalty0 331--368, 2007.

\bibitem[Carratino et~al.(2018)Carratino, Rudi, and
  Rosasco]{carratino2018learning}
L.~Carratino, A.~Rudi, and L.~Rosasco.
\newblock Learning with sgd and random features.
\newblock In \emph{Advances in Neural Information Processing Systems}, pages
  10192--10203, 2018.

\bibitem[Chen et~al.(2019)Chen, Xu, and Shrivastava]{chen2019fast}
B.~Chen, Y.~Xu, and A.~Shrivastava.
\newblock Fast and accurate stochastic gradient estimation.
\newblock In \emph{Advances in Neural Information Processing Systems}, pages
  12339--12349, 2019.

\bibitem[Csiba and Richt{\'a}rik(2018)]{csiba2018importance}
D.~Csiba and P.~Richt{\'a}rik.
\newblock Importance sampling for minibatches.
\newblock \emph{The Journal of Machine Learning Research}, 19\penalty0
  (1):\penalty0 962--982, 2018.

\bibitem[Dasgupta(2005)]{dasgupta2005analysis}
S.~Dasgupta.
\newblock Analysis of a greedy active learning strategy.
\newblock In \emph{Advances in neural information processing systems}, pages
  337--344, 2005.

\bibitem[Dieuleveut et~al.(2016)Dieuleveut, Bach,
  et~al.]{dieuleveut2016nonparametric}
A.~Dieuleveut, F.~Bach, et~al.
\newblock Nonparametric stochastic approximation with large step-sizes.
\newblock \emph{The Annals of Statistics}, 44\penalty0 (4):\penalty0
  1363--1399, 2016.

\bibitem[Gal et~al.(2017)Gal, Islam, and Ghahramani]{gal2017deep}
Y.~Gal, R.~Islam, and Z.~Ghahramani.
\newblock Deep bayesian active learning with image data.
\newblock In \emph{Proceedings of the 34th International Conference on Machine
  Learning-Volume 70}, pages 1183--1192. JMLR. org, 2017.

\bibitem[Guo and Schuurmans(2008)]{guo2008discriminative}
Y.~Guo and D.~Schuurmans.
\newblock Discriminative batch mode active learning.
\newblock In \emph{Advances in neural information processing systems}, pages
  593--600, 2008.

\bibitem[Jacot et~al.(2018)Jacot, Gabriel, and Hongler]{jacot2018neural}
A.~Jacot, F.~Gabriel, and C.~Hongler.
\newblock Neural tangent kernel: Convergence and generalization in neural
  networks.
\newblock In \emph{Advances in neural information processing systems}, pages
  8571--8580, 2018.

\bibitem[Ji et~al.(2012)Ji, Yang, Lin, Jin, and Han]{ji2012simple}
M.~Ji, T.~Yang, B.~Lin, R.~Jin, and J.~Han.
\newblock A simple algorithm for semi-supervised learning with improved
  generalization error bound.
\newblock \emph{arXiv preprint arXiv:1206.6412}, 2012.

\bibitem[Jun et~al.(2019)Jun, Cutkosky, and Orabona]{jun2019kernel}
K.-S. Jun, A.~Cutkosky, and F.~Orabona.
\newblock Kernel truncated randomized ridge regression: Optimal rates and low
  noise acceleration.
\newblock In \emph{Advances in Neural Information Processing Systems}, pages
  15332--15341, 2019.

\bibitem[Kapoor et~al.(2007)Kapoor, Grauman, Urtasun, and
  Darrell]{kapoor2007active}
A.~Kapoor, K.~Grauman, R.~Urtasun, and T.~Darrell.
\newblock Active learning with gaussian processes for object categorization.
\newblock In \emph{2007 IEEE 11th International Conference on Computer Vision},
  pages 1--8. IEEE, 2007.

\bibitem[Lin and Rosasco(2017)]{lin2017optimal}
J.~Lin and L.~Rosasco.
\newblock Optimal rates for multi-pass stochastic gradient methods.
\newblock \emph{The Journal of Machine Learning Research}, 18\penalty0
  (1):\penalty0 3375--3421, 2017.

\bibitem[Needell et~al.(2014)Needell, Ward, and Srebro]{needell2014stochastic}
D.~Needell, R.~Ward, and N.~Srebro.
\newblock Stochastic gradient descent, weighted sampling, and the randomized
  kaczmarz algorithm.
\newblock In \emph{Advances in neural information processing systems}, pages
  1017--1025, 2014.

\bibitem[Pillaud-Vivien et~al.(2018)Pillaud-Vivien, Rudi, and
  Bach]{pillaud2018statistical}
L.~Pillaud-Vivien, A.~Rudi, and F.~Bach.
\newblock Statistical optimality of stochastic gradient descent on hard
  learning problems through multiple passes.
\newblock In \emph{Advances in Neural Information Processing Systems}, pages
  8114--8124, 2018.

\bibitem[Rahimi and Recht(2008)]{rahimi2008random}
A.~Rahimi and B.~Recht.
\newblock Random features for large-scale kernel machines.
\newblock In \emph{Advances in neural information processing systems}, pages
  1177--1184, 2008.

\bibitem[Rigollet(2007)]{rigollet2007generalization}
P.~Rigollet.
\newblock Generalization error bounds in semi-supervised classification under
  the cluster assumption.
\newblock \emph{Journal of Machine Learning Research}, 8\penalty0
  (Jul):\penalty0 1369--1392, 2007.

\bibitem[Roh et~al.(2019)Roh, Heo, and Whang]{roh2019survey}
Y.~Roh, G.~Heo, and S.~E. Whang.
\newblock A survey on data collection for machine learning: a big data-ai
  integration perspective.
\newblock \emph{IEEE Transactions on Knowledge and Data Engineering}, 2019.

\bibitem[Rosasco and Villa(2015)]{rosasco2015learning}
L.~Rosasco and S.~Villa.
\newblock Learning with incremental iterative regularization.
\newblock In \emph{Advances in Neural Information Processing Systems}, pages
  1630--1638, 2015.

\bibitem[Rudi and Rosasco(2017)]{rudi2017generalization}
A.~Rudi and L.~Rosasco.
\newblock Generalization properties of learning with random features.
\newblock In \emph{Advances in Neural Information Processing Systems}, pages
  3215--3225, 2017.

\bibitem[Sch{\"o}lkopf et~al.(2002)Sch{\"o}lkopf, Smola, Bach,
  et~al.]{scholkopf2002learning}
B.~Sch{\"o}lkopf, A.~J. Smola, F.~Bach, et~al.
\newblock \emph{Learning with kernels: support vector machines, regularization,
  optimization, and beyond}.
\newblock MIT press, 2002.

\bibitem[Sener and Savarese(2017)]{sener2017active}
O.~Sener and S.~Savarese.
\newblock Active learning for convolutional neural networks: A core-set
  approach.
\newblock \emph{arXiv preprint arXiv:1708.00489}, 2017.

\bibitem[Settles(2009)]{settles2009active}
B.~Settles.
\newblock Active learning literature survey.
\newblock Technical report, University of Wisconsin-Madison Department of
  Computer Sciences, 2009.

\bibitem[Steinwart et~al.(2009)Steinwart, Hush, Scovel,
  et~al.]{steinwart2009optimal}
I.~Steinwart, D.~R. Hush, C.~Scovel, et~al.
\newblock Optimal rates for regularized least squares regression.
\newblock In \emph{COLT}, pages 79--93, 2009.

\bibitem[Vapnik and Vapnik(1998)]{vapnik1998statistical}
V.~Vapnik and V.~Vapnik.
\newblock Statistical learning theory wiley.
\newblock \emph{New York}, 1, 1998.

\bibitem[Wasserman and Lafferty(2008)]{wasserman2008statistical}
L.~Wasserman and J.~D. Lafferty.
\newblock Statistical analysis of semi-supervised regression.
\newblock In \emph{Advances in Neural Information Processing Systems}, pages
  801--808, 2008.

\bibitem[Wei et~al.(2015)Wei, Iyer, and Bilmes]{wei2015submodularity}
K.~Wei, R.~Iyer, and J.~Bilmes.
\newblock Submodularity in data subset selection and active learning.
\newblock In \emph{International Conference on Machine Learning}, pages
  1954--1963, 2015.

\bibitem[Xiao et~al.(2017)Xiao, Rasul, and Vollgraf]{xiao2017fashion}
H.~Xiao, K.~Rasul, and R.~Vollgraf.
\newblock Fashion-mnist: a novel image dataset for benchmarking machine
  learning algorithms.
\newblock \emph{arXiv preprint arXiv:1708.07747}, 2017.

\bibitem[Yao et~al.(2007)Yao, Rosasco, and Caponnetto]{yao2007early}
Y.~Yao, L.~Rosasco, and A.~Caponnetto.
\newblock On early stopping in gradient descent learning.
\newblock \emph{Constructive Approximation}, 26\penalty0 (2):\penalty0
  289--315, 2007.

\bibitem[Yu et~al.(2006)Yu, Bi, and Tresp]{yu2006active}
K.~Yu, J.~Bi, and V.~Tresp.
\newblock Active learning via transductive experimental design.
\newblock In \emph{Proceedings of the 23rd international conference on Machine
  learning}, pages 1081--1088, 2006.

\bibitem[Zhang(2005)]{zhang2005learning}
T.~Zhang.
\newblock Learning bounds for kernel regression using effective data
  dimensionality.
\newblock \emph{Neural Computation}, 17\penalty0 (9):\penalty0 2077--2098,
  2005.

\bibitem[Zhao and Zhang(2015)]{zhao2015stochastic}
P.~Zhao and T.~Zhang.
\newblock Stochastic optimization with importance sampling for regularized loss
  minimization.
\newblock In \emph{international conference on machine learning}, pages 1--9,
  2015.

\end{thebibliography}
\bibliographystyle{abbrvnat}
\onecolumn
\appendix
\section{Auxiliary Results}
First we introduce GD path on the excess risk:
\begin{align*}
    f_t =&\ f_{t-1} - \eta \nabla \mathcal{E}(f_{t-1}) \\
    =&\ f_{t-1} - \eta \int_\mathcal{Z} (\langle f_{t-1}, K_x\rangle_H - y)K_x d\rho_\mathcal{Z}(x, y) \\
    =&\ f_{t-1} - \eta (\Sigma \mu_{t-1} - S^*f_*)
\end{align*}
with $f_0 = 0 \in H$ for $t \in \mathbb{N}$. 

\begin{lemma}[Proposition 2 and Extension of Lemma 16 in \cite{lin2017optimal}]\label{lem: f_boundedness}
Suppose that Assumptions \ref{assump: kernel_boundedness} and \ref{assump: smoothness_of_true} hold.
Let $\eta = O(1/\kappa^2)$ be sufficiently small. Then, for any $t \in \mathbb{N}$,
\begin{align*}
    \|Sf_t - f_*\|_{L^2(\rho_\mathcal X)}^2 = O\left(R^2\left(\eta t)^{-2r}\right)\right).
\end{align*}
Moreover for any $\lambda > 0$ and $s \in [0, r]$
\begin{align*}
    \|\Sigma_\lambda^{-s}f_t\|_H^2 \leq O(R^2(\kappa^{4(r-s)-2} + (\eta t)^{1-2(r-s)})).
\end{align*}
\end{lemma}
\begin{proof}
The second statement is a straight forward extension of Lemma 16 in \cite{lin2017optimal} and we omit it.
\end{proof}
\begin{lemma}\label{lem: effective_dimensionality}
Suppose that Assumptions \ref{assump: kernel_boundedness} and \ref{assump: eigen_decay} hold. For any $\lambda > 0$, 
\begin{align*}
    \mathcal N_\infty(\lambda) \stackrel{\mathrm{def}}{=} \mathbb{E}_x\|\Sigma_\lambda^{-1/2}K_x\|_H^2 = \mathrm{Tr}(\Sigma_\lambda^{-1}\Sigma) \leq \mathrm{Tr}(\Sigma^{\frac{1}{\alpha}})\lambda^{-\frac{1}{\alpha}}.
\end{align*}
\end{lemma}

\begin{proof}
$\mathbb{E}_x\|\Sigma_\lambda^{-1/2}K_x\|_H^2 = \mathbb{E}_x\mathrm{Tr}(\Sigma_\lambda^{-1/2}(K_x\otimes K_x) \Sigma_\lambda^{-1/2}) = \mathrm{Tr}(\Sigma_\lambda^{-1}\Sigma)$.
Observe that $\mathrm{Tr}(\Sigma_\lambda^{-1}\Sigma) = \sum_{i=1}^\infty \lambda_i/(\lambda_i+\lambda) \leq \sum_{i=1}^\infty (\lambda_i/(\lambda_i+\lambda))^{1/\alpha} \leq (\sum_{i=1}^\infty \lambda_i^{1/\alpha})\lambda^{-1/\alpha} = \mathrm{Tr}(\Sigma^{1/\alpha})\lambda^{-1/\alpha}$. This finishes the proof.
\end{proof}

\begin{lemma}\label{lem: compatibility}
Suppose that Assumption \ref{assump: kernel_boundedness} holds. For any $\lambda > 0$, 
\begin{align*}
    \mathcal F_\infty(\lambda) \stackrel{\mathrm{def}}{=} \mathrm{sup}_{x \in \mathrm{supp}(\rho_{\mathcal X})}\|\Sigma_\lambda^{-1/2}K_x\|_H^2 = \kappa^2\lambda^{-1}.
\end{align*}
\end{lemma}
\begin{proof}
From Assumptions \ref{assump: kernel_boundedness}, we immediately obtain the claim.
\end{proof}

\begin{lemma}[Spectral filters]\label{lem: spectral_filters}
Let $p_t(x) = \eta\sum_{k=0}^{t} (1 - \eta x)^{k}$ for $x \in [0, 1/\eta)$ and $t \in \mathbb{N}$. Also we define $r_t(x) = 1 - xp_t(x)$ $x \in [0, 1/\eta)$ and $t \in \mathbb{N}$. Then the following inequalities hold:
\begin{align*}
    \mathrm{sup}_{x \in [0, 1/\eta)}p_t(x)x \leq O(1)
\end{align*}
for any $t \in \mathbb{N}$ and 
\begin{align*}
    \mathrm{sup}_{x \in [0, 1/\eta)}r_t(x)x^u \leq O(1)(\eta t)^{-u}
\end{align*}
for any $t \in \mathbb{N}$ and $u \in [0, 1]$.
\end{lemma}
\begin{proof}
When $x = 0$, the inequalities always hold and so we assume $x > 0$. Note that $p_t(x) = (1 - (1-\eta x)^t)/x$. The first inequality is trivial because $1 - (1 - \eta x)^t \leq 1$. We show the second inequality. Note that $r_t(x) = (1 - \eta x)^t$. Observe that from elemental calculus, function $(1 - \eta x)^t x^u$ for $x \in (0, 1/\eta)$ is maximized at $x = u/(\eta (u+t))$ and has maximum value $u^u/(\eta^u (u+t)^u) (1 - u/(u+t))^t \leq u^u (\eta t)^{-u}$. This finishes the proof.   
\end{proof}

Recall that
\begin{align*}
    q_j = \frac{\|\Sigma_{N, \lambda_q}^{-\frac{1}{2}}K_{x_j}\|_H^2+\frac{1}{N}\sum_{i=1}^N \|\Sigma_{N, \lambda_q}^{-\frac{1}{2}}K_{x_j}\|_H^2}{2\sum_{j=1}^N\|\Sigma_{N, \lambda_q}^{-\frac{1}{2}}K_{x_j}\|_H^2}
\end{align*}
for $j \in [N]$. $\lambda_q$ will be set to $\lambda_*$, where $1/(\eta \lambda_*)$ is the optimal number of iterations (see Definition \ref{def: opt_num_iter} in the main paper)). Then we define $\Sigma_n^{(q)} = (1/n)\sum_{i=1}^n 1/(Nq_{j(i)}))K_{x_{j(i)}}\otimes K_{x_{j(i)}}$, where $j(i)$ is uniformly at random on $[N]$. 

\begin{lemma}\label{lem: emp_exp_prod}
Suppose that Assumption \ref{assump: kernel_boundedness}. Let $s \in [0, 1/2)$ and $\delta \in (0, 1]$. Suppose that $\lambda = \Omega((\mathrm{Tr}(\Sigma^{1/\alpha})/n)^\alpha)$. When $N = \Omega(1+\kappa^2\lambda^{-1}\mathrm{log}(n/\delta)))$, with probability at least $1-\delta$
\begin{align*}
    \left\|\Sigma^{\frac{1}{2}-s}\Sigma_{N, \lambda}^{-\frac{1}{2}}\right\| = O(\lambda^{-s})
\end{align*}
holds. 
\end{lemma} 
\begin{proof}
For $\lambda > \|\Sigma\|$, the claim is trivial. 
Note that since $s \in [0, 1/2)$, $\|\Sigma^{1/2-s}\Sigma_{N, \lambda}^{-1/2}\| \leq \|\Sigma^{1/2-s}_\lambda \Sigma_{N, \lambda}^{-1/2}\| \leq \lambda^{-s}\|\Sigma^{1/2}_\lambda \Sigma_{N, \lambda}^{-1/2}\| = \lambda^{-s}(1+\|\Sigma^{1/2}\Sigma_{N, \lambda}^{-1/2}\|)$. Then from Proposition 8 in \citep{rudi2017generalization}, we have $\|\Sigma^{1/2}\Sigma_{N, \lambda}^{-1/2}\| \leq \{1 - \lambda_{\mathrm{max}}(\Sigma_\lambda^{-1/2}(\Sigma^{1/2} - \Sigma_N)\Sigma_\lambda^{-1/2})\}^{-1/2}$. Now from Proposition 6 in \citep{rudi2017generalization}\footnotemark, we have with probability $1-\delta$,
\begin{align}
    &\lambda_{\mathrm{max}}(\Sigma_\lambda^{-\frac{1}{2}}(\Sigma - \Sigma_N)\Sigma_\lambda^{-\frac{1}{2}}) \notag \\
    =& O(1)\left(\sqrt{\frac{ \mathrm{log}\left(\mathrm{Tr}(\Sigma^{\frac{1}{\alpha}})\lambda^{-\frac{1}{\alpha}}\delta^{-1}\right)\mathcal F_\infty(\lambda)}{N}}+\frac{\mathrm{log}\left(\mathrm{Tr}(\Sigma^{\frac{1}{\alpha}})\lambda^{-\frac{1}{\alpha}}\delta^{-1}\right)}{N}\right) \label{eq: emp_cov_diff}
\end{align}
for any $\lambda \in (0, \|\Sigma\|]$. Assume $\lambda = \Omega((\mathrm{Tr}(\Sigma^{1/\alpha})/n)^\alpha$, By using Lemma \ref{lem: compatibility}, we can see that r.h.s of (\ref{eq: emp_cov_diff}) becomes smaller than $0.5$ when $N=\Omega(1+\kappa^2\lambda^{-1}\mathrm{log}(n/\delta))$. Then we obtain the desired result. 
\end{proof}
\footnotetext{Proposition 6 in \citep{rudi2017generalization}, the logarithmic factor in (\ref{eq: emp_cov_diff}) is replaced with $\mathrm{log}(\mathrm{Tr}(\Sigma)\lambda^{-1}\delta^{-1})$. This is due to loose bound $\mathrm{Tr}(\Sigma_\lambda^{-1}\Sigma) \leq \mathrm{Tr}(\Sigma)\lambda^{-1}$ in their proof and we can improve the bound to $\mathrm{Tr}(\Sigma_\lambda^{-1}\Sigma) \leq \mathrm{Tr}(\Sigma^{1/\alpha})\lambda^{-1/\alpha}$ from Lemma \ref{lem: effective_dimensionality}. This improvement is important for extremely small $\lambda$. For example, when $\lambda = \mathrm{Tr}(\Sigma^{1/\alpha})/n)^\alpha$, that is the lower bound of $\lambda$ in our theory, the loose log factor becomes $\alpha\mathrm{log}(n)$ rather than $\mathrm{log}(n)$, which goes to $\infty$ as $\alpha \to \infty$. }

\begin{lemma}\label{lem: is_emp_exp_prod}
Suppose that Assumptions \ref{assump: kernel_boundedness} and \ref{assump: eigen_decay} hold. Let $\delta \in (0, 1]$ and $\lambda \geq \lambda_q = \Omega((\mathrm{Tr}(\Sigma^{1/\alpha})/n)^\alpha)$. When $n =  \Omega(1 + \mathrm{Tr}(\Sigma^{-1/\alpha})\lambda_q^{-1/\alpha}\mathrm{log}^2(n/\delta))$ and $N = \Omega(1 + \kappa^2 \lambda_q^{-1}\mathrm{log}(n/\delta))$, with probability $1-3\delta$
\begin{align*}
    \left\|\Sigma_N^{\frac{1}{2}}(\Sigma_{n, \lambda}^{(q)})^{-\frac{1}{2}}\right\| = O(1)
\end{align*}
holds. 
\end{lemma} 
\begin{proof}
The proof is similar to the one of Lemma \ref{lem: emp_exp_prod}. Suppose that $X_N$ is given. At first, for $\lambda > \|\Sigma_N\|$, the claim is trivial. 
Next, from Proposition 8 in \citep{rudi2017generalization}, we have $\|\Sigma_N^{1/2} (\Sigma_{n, \lambda}^{(q)})^{-1/2}\| \leq \{1 - \lambda_{\mathrm{max}}(\Sigma_\lambda^{-1/2}(\Sigma_N^{1/2} - \Sigma_{n, \lambda}^{(q)})\Sigma_\lambda^{-1/2})\}^{-1/2}$. 
Recall that $\mathbb{E}[\Sigma_{n, \lambda}^{(q)} | X_N] = \Sigma_N$. 
Observe that 
\begin{align*}
    \left\|\Sigma_{N, \lambda}^{-\frac{1}{2}}\frac{1}{\sqrt{Nq_{j(i)}}}K_{x_{j(i)}}\right\|_H^2 
    =&\ \frac{1}{Nq_{j(i)}}\left\|\Sigma_{N, \lambda}^{-\frac{1}{2}}K_{x_{j(i)}}\right\|_H^2 \\
    \leq&\ \frac{1}{Nq_{j(i)}}\left\|\Sigma_{N, \lambda_q}^{-\frac{1}{2}}K_{x_{j(i)}}\right\|_H^2 \\ 
    \leq&\ \frac{2}{N}\sum_{j=1}^N \left\|\Sigma_{N, \lambda_q}^{-\frac{1}{2}} K_{x_j}\right\|_H^2 \\
    =&\ \frac{2}{N}\sum_{j=1}^N \left\|\Sigma_{N, \lambda_q}^{-\frac{1}{2}}\Sigma_{\lambda_q}^{\frac{1}{2}}\right\|^2\left\|\Sigma_{\lambda_q}^{-\frac{1}{2}} K_{x_j}\right\|_H^2 \\
    =&\ O(1)\frac{1}{N}\sum_{j=1}^N\left\|\Sigma_{\lambda_q}^{-\frac{1}{2}} K_{x_j}\right\|_H^2,
\end{align*}
where the second inequality holds from $\lambda \geq \lambda_q$ and the last inequality holds from Lemma \ref{lem: emp_exp_prod}. 
Similar to the arguments in the proof of Lemma \ref{lem: emp_exp_prod}, for any $\lambda \in (0, \|\Sigma_N\|]$, we have with probability at least $1-2\delta$, 
\begin{align*}
    &\lambda_{\mathrm{max}}\left(\Sigma_\lambda^{-\frac{1}{2}}(\Sigma_N - \Sigma_{n, \lambda}^{(q)})\Sigma_\lambda^{-\frac{1}{2}}\right) \\
    =&\ O(1)\left(\sqrt{\frac{\mathrm{log}(\mathrm{Tr}(\Sigma^\frac{1}{\alpha})\lambda^{-\frac{1}{\alpha}}\delta^{-1})\frac{1}{N}\sum_{j=1}^N \|\Sigma_{ \lambda_q}^{-\frac{1}{2}}K_{x_j}\|^2}{n}} + \frac{\mathrm{log}(\mathrm{Tr}(\Sigma^\frac{1}{\alpha})\lambda^{-\frac{1}{\alpha}}\delta^{-1})}{n}\right)
\end{align*}
given $\{x_j\}_{j=1}^N$. Here we used that fact that $\mathrm{Tr}(\Sigma_{N, \lambda}^{-1}\Sigma_N)  = O(1) \mathrm{Tr}(\Sigma_{ \lambda}^{-1}\Sigma)$ with probability at least $1 - \delta$ if $N = \Omega(1 + \kappa^2\lambda^{-1}\mathrm{log}(n/\delta))$ from the similar results to Proposition 10 in \citep{rudi2017generalization}\footnotemark.
\footnotetext{Here, we need to note two things. First, Proposition 10 in \citep{rudi2017generalization} bounds the empirical effective dimension from {\it{random features}}. In contrast, our bound is based on the empirical effective dimension from the observed input data. However, the proof is perfectly similar to the random features cases and we omit it. Second, The proof of Proposition 10 in \citep{rudi2017generalization} relies on Proposition 6 in \citep{rudi2017generalization}. Thus, an improved logarithmic factor in the bound can be obtained as we noted in the footnote of the proof of Lemma \ref{lem: emp_exp_prod}.} 

Then, using standard Bernstein's inequality for i.i.d. random variables $\{\|\Sigma_{\lambda_*}^{-1/2} K_{x_j}\|^2\}_{j=1}^N$, we have
\begin{align*}
    \frac{1}{N}\sum_{j=1}^N \left\|\Sigma_{ \lambda_q}^{-\frac{1}{2}}K_{x_j}\right\|^2 =&\ O(1)\left(\mathcal N_\infty(\lambda_q) + \sqrt{\frac{\mathrm{log}(\delta^{-1})\mathcal N_\infty(\lambda_q) \mathcal F_\infty(\lambda_q)}{N}} + \frac{\mathrm{log}(\delta^{-1})\mathcal F_\infty(\lambda_q)}{N}\right) \\
    =&\ O(1)\left(\mathrm{Tr}(\Sigma^{\frac{1}{\alpha}})\lambda^{-\frac{1}{\alpha}} + \sqrt{\frac{\mathrm{log}(\delta^{-1})\mathrm{Tr}(\Sigma^{\frac{1}{\alpha}})\kappa^2\lambda_q^{-\frac{1}{\alpha}-1}}{N}} + \frac{\mathrm{log}(\delta^{-1})\kappa^2\lambda_q^{-1}}{N}\right) \\
    =&\  O(1)\left(\mathrm{Tr}(\Sigma^{\frac{1}{\alpha}})\lambda_q^{-\frac{1}{\alpha}}\mathrm{log}(\delta^{-1}) + \frac{\kappa^2\lambda_q^{-1}\mathrm{log}(\delta^{-1})}{N}\right)
\end{align*}
with probability at least $1-\delta$. For the second inequality, we used Lemma \ref{lem: effective_dimensionality} and \ref{lem: compatibility}. The last inequality holds due to inequality of arithmetic and geometric means. \par
Combining all the results, with probability at least $1 - 3\delta$, 
when $n = \Omega(1 + \mathrm{Tr}(\Sigma^{-1/\alpha})\lambda_q^{-1/\alpha}\mathrm{log}^2(n/\delta))$ and $N = \widetilde{\Omega}(1 + \kappa^2\lambda_q^{-1}\mathrm{log}(n/\delta))$, we obtain the claim of Lemma \ref{lem: is_emp_exp_prod}. 
\end{proof}

The following lemma is essential for our analysis:
\begin{lemma}\label{lem: is_grad_diff}
Suppose that Assumptions \ref{assump: kernel_boundedness}, \ref{assump: eigen_decay} and \ref{assump: noise} hold. For $\delta \in (0, 1)$, $\lambda \geq \lambda_q$ and $t \in \mathbb{N}$, there exists event $A$ such that 
\begin{align*}
    &\mathbb{E}\left[\left\|\Sigma_\lambda^{-\frac{1}{2}}(\Sigma_n^{(q)} f_t - (S_n^{(q)})^*\bm y_n - (\Sigma f_t - S^*f_*))\right\|_H^2 \mid A\right] \\
    =&\  \widetilde{O}\left(\frac{\mathrm{Tr}(\Sigma^{\frac{1}{\alpha}})(\sigma^2+R^2(\eta t)^{-2r})\lambda_q^{-\frac{1}{\alpha}}}{n} + \frac{\kappa^2\lambda_q^{-1}}{nN}\left(\sigma^2 + R^2(\eta t)^{-2r} + \frac{M^2+\kappa^{4r-2}R^2+R^2(\eta t)^{1-2r}}{N}\right)\right).
\end{align*}
Here $\widetilde{O}$ hides extra $\mathrm{poly}( \mathrm{\delta^{-1}}))$ factors.
\end{lemma}
\begin{proof}
Let $\zeta_i = \Sigma_\lambda^{-1/2}1/(Nq_{j(i)})(K_{x_{j(i)}}\otimes K_{x_{j(i)}}f_t - y_j(i) K_{x_{j(i)}}) = \Sigma_\lambda^{-1/2}1/(Nq_{j(i)})K_{x_{j(i)}}(f_t(x_{j(i)}) - y_{j(i)})$ for $i \in [n]$. Since $\{\zeta_i\}_{i=1}^N$ is i.i.d. sequence and $\mathbb{E}[\zeta_i|X_N] = \Sigma_\lambda^{-1/2}(\Sigma_N f_t - S_N^*f_*)$, we have
\begin{align*}
    &\mathbb{E}\left\|\Sigma_\lambda^{-\frac{1}{2}}(\Sigma_n^{(q)} f_t - (S_n^{(q)})^*\bm y_n - (\Sigma f_t - S^*f_*))\right\|_H^2 \\
    =&\ \mathbb{E}\left\|\Sigma_\lambda^{-\frac{1}{2}}(\Sigma_n^{(q)} f_{k-1} - (S_n^{(q)})^*\bm y_n - (\Sigma_N f_t - S_N^*\bm y_N))\right\|_H^2 \\
    &+ \mathbb{E}\|\Sigma_\lambda^{-\frac{1}{2}}(\Sigma_N f_{k-1} - S_N^*\bm y_N - (\Sigma f_t - S^*f_*))\|_H^2 \\
    \leq&\ \frac{1}{n}\mathbb{E}\|\xi_i\|_H^2 + \mathbb{E}\|\Sigma_\lambda^{-\frac{1}{2}}(\Sigma_N f_t - S_N^*\bm y_N - (\Sigma f_{k-1} - S^*f_*))\|_H^2.
\end{align*}

Observe that
\begin{align*}
    \mathbb{E}[\|\zeta_i\|_H^2|X_N] = \left(\frac{1}{N}\sum_{i=1}^N \|\Sigma_{ \lambda_q}^{-\frac{1}{2}}K_{x_i}\|_H^2\right)\left(\frac{1}{N}\sum_{i=1}^N (y_i - f_t(x_i))^2\right)
\end{align*}
from the definition of $q$ and Lemma \ref{lem: is_emp_exp_prod} and \ref{lem: emp_exp_prod}. Hence we have
\begin{align*}
    \mathbb{E}\|\zeta_i\|_H^2 = \mathbb{E}\left[\left(\frac{1}{N}\sum_{i=1}^N \|\Sigma_{\lambda_q}^{-\frac{1}{2}}K_{x_i}\|_H^2\right)\left(\frac{1}{N}\sum_{i=1}^N (y_i - f_t(x_i))^2\right)\right].
\end{align*}

Now, similar to the arguments in the proof of Lemma \ref{lem: is_emp_exp_prod}, since
\begin{align*}
    \frac{1}{N}\sum_{j=1}^N \|\Sigma_{ \lambda_q}^{-\frac{1}{2}}K_{x_j}\|^2
    \leq&\ \kappa^2\lambda_q^{-1}
\end{align*}
a.s. from Lemma \ref{lem: compatibility}, 
we have
\begin{align}\label{ineq: tight_compatibiliy}
    \frac{1}{N}\sum_{j=1}^N \|\Sigma_{ \lambda_q}^{-\frac{1}{2}}K_{x_j}\|^2
    \leq&\ O\left(\mathrm{Tr}(\Sigma^{\frac{1}{\alpha}})\lambda_q^{-\frac{1}{\alpha}}\mathrm{log}(\delta^{-1}) + \frac{\kappa^2\lambda_q^{-1}\mathrm{log}(\delta^{-1})}{N}\right)
\end{align}
with probability at least $1-\delta$. Also, from Bernstein's inequality and Assumption \ref{assump: noise} with probability at least $1-1/\delta$ it holds that
\begin{align}\label{ineq: obj_gap_of_gd_on_excess}
    \frac{1}{N}\sum_{j=1}^N (y_i - f_{t}(x_i))^2
    \leq&\ O\left((\sigma^2 + R^2(\eta t)^{-2r})\mathrm{log}(\delta^{-1}) + \frac{(M^2+\kappa^{4r-2}R^2 + R^2(\eta t)^{1-2r})\mathrm{log}(\delta^{-1})}{N}\right)
\end{align}
because 
\begin{align*}
    \mathbb{E}_{(x_i, y_i)}[(y_i - f_*(x_i))^2] = \sigma^2,
\end{align*}
\begin{align*}
    \mathbb{E}_{x_i}[(f_t(x_i) - f_*(x_i))^2] \leq O(R^2(\eta t)^{-2r}),
\end{align*}
\begin{align*}
    (y_i - f_*(x_i))^2 = M^2
\end{align*}
and
\begin{align*}
    (f_*(x_i) - f_t(x_i))^2 \leq O\left(M^2 + \|f_{k-1}\|_H^2\right) \leq O\left(M^2 + \kappa^{4r - 2}\kappa^{4r-2} + R^2(\eta t)^{1-2r}\right)
\end{align*}
a.s. from Lemma \ref{lem: f_boundedness}.

Denote the event $A$ that satisfies (\ref{ineq: tight_compatibiliy}) and (\ref{ineq: obj_gap_of_gd_on_excess}).
Then $P(A) \geq 1 - 2\delta$ and
\begin{align*}
    \mathbb{E}[\|\xi_i\|^2|A] 
    \leq&\  O(\mathrm{log}^2(\delta^{-1}))\left(\mathrm{Tr}(\Sigma^{\frac{1}{\alpha}})\lambda_q^{-\frac{1}{\alpha}} + \frac{\lambda_q^{-1}}{N}\right)\left(\sigma^2 + (\eta k)^{-2r} + \frac{(M^2+(\eta k)^{1-2r})}{N}\right).
\end{align*}

Finally, $\mathbb{E}\|\Sigma_\lambda^{-1/2}(\Sigma_N f_t - S_N^*\bm y_N - (\Sigma f_t - S^*f_*))\|_H^2$ can be bounded by
\begin{align*}
    O\left(\frac{\mathrm{log}(\delta^{-1})}{N}\left(\sigma^2 \mathrm{Tr}(\Sigma^{-\frac{1}{\alpha}})\lambda^{-\frac{1}{\alpha}} + \lambda^{-1}(\eta t)^{-2r}\right)\right).  
\end{align*}

Combining these results, we obtain the desired inequality.
\end{proof}

\begin{lemma}\label{lem: excess_diff}
Let $\eta = O(1/\kappa^2)$ be sufficiently small. For any $t \in \mathbb{N}$, 
\begin{align*}
    \left\|\Sigma_{\frac{1}{\eta t}}^{-\frac{1}{2}}(\Sigma f_t - S^*f_*)\right\|_H^2 = O(R^2(\eta t)^{-2r})
\end{align*}
\end{lemma}
\begin{proof}
We denote $\lambda = 1/(\eta t)$.
Note that $\Sigma f_t - S^*f_* = r_t(\Sigma)S^*f_* = r_t(\Sigma)S^*\mathcal L^r \phi$ for some $phi \in L^2(\rho_{\mathcal X})$. Then we have
\begin{align*}
    \|\Sigma_\lambda^{-\frac{1}{2}}(\Sigma f_t - S^*f_*)\|_H \leq \|\Sigma^r r_t(\Sigma)\|
    \|\Sigma_\lambda^{-\frac{1}{2}}S^*\|\|\phi\|_{L^2(\rho_{\mathcal X})}.
\end{align*}

From Lemma \ref{lem: spectral_filters}, we have
\begin{align*}
     \|\Sigma^r r_t(\Sigma)\| \leq O(\lambda^r).
\end{align*}
Also, observe that
\begin{align*}
    \left\|\Sigma_\lambda^{-\frac{1}{2}}S^*\right\| = \left\|\Sigma_\lambda^{-\frac{1}{2}}\Sigma \Sigma_\lambda^{-\frac{1}{2}}\right\|^\frac{1}{2} \leq 1.
\end{align*}
Combining these results finishes the proof.
\end{proof}

\section{Proof of Main Results} \label{app_sec: main_results}
First we decompose the error $\|\Sigma^{1/2-s}(g_t - f_t)\|_H^2$ to two terms: 
\begin{align*}
    &\left\|\Sigma^{\frac{1}{2}-s}(g_t - f_t)\right\|_H^2 \\
    \leq&\ 2\left\|\Sigma^{\frac{1}{2}-s}(g_t - p_t(\Sigma_n^{(q)
    })\Sigma_n^{(q)}f_t)\right\|_H^2
    + 2\left\|\Sigma^{\frac{1}{2}-s}(p_t(\Sigma_n^{(q)
    })\Sigma_n^{(q)}f_t - f_t)\right\|_H^2. 
\end{align*}

The first term can be bounded as follows:
\begin{align}
    &\left\|\Sigma^{\frac{1}{2}-s}(g_t - p_t(\Sigma_n^{(q)
    })\Sigma_n^{(q)}f_t)\right\|_H^2 \notag \\
    =&\ \left\|\Sigma^{\frac{1}{2}-s}(p_t(\Sigma_n^{(q)})(S_n^{(q)})^*\bm y_n - p_t(\Sigma_n^{(q)
    })\Sigma_n^{(q)}f_t)\right\|_H^2 \notag \\
    \leq&\ \left\|\Sigma^{\frac{1}{2}-s}\Sigma_{N, \lambda}^{-\frac{1}{2}}\right\|^2
    \left\|\Sigma_{N,\lambda}^{\frac{1}{2}}(\Sigma_{n, \lambda}^{(q)})^{-\frac{1}{2}}\right\|^2
    \left\|(\Sigma_{n, \lambda}^{(q)})p_t(\Sigma_n^{(q)})\right\|^2
    \left\|(\Sigma_{n, \lambda}^{(q)})^{-\frac{1}{2}}\Sigma_{N, \lambda}^\frac{1}{2}\right\|^2 \notag \\
    &\ \ \ \ \ \ \ \ \ \ \ \times \left\|\Sigma_{N, \lambda}^{-\frac{1}{2}}\Sigma_\lambda^{\frac{1}{2}}\right\|^2
    \left\|\Sigma_\lambda^{-\frac{1}{2}}((S_n^{(q)})^*\bm y_n - \Sigma_n^{(q)}f_t)\right\|_H^2 \label{ineq: decomposed_two_errors1}
\end{align}
for any $\lambda > 0$.
The second term has following bound:
\begin{align}
    &\left\|\Sigma^{\frac{1}{2}-s}(p_t(\Sigma_n^{(q)
    })\Sigma_n^{(q)}f_t - f_t)\right\|_H^2 \notag \\
    \leq&\ \left\|\Sigma^{\frac{1}{2}-s} \Sigma_{N, \lambda}^{-\frac{1}{2}}\right\|^2
    \left\|\Sigma_{N, \lambda}^\frac{1}{2}(\Sigma_\lambda^{(q)})^{-\frac{1}{2}}\right\|^2
    \left\|(\Sigma_\lambda^{(q)})^{\frac{1}{2} \vee r} r_t(\Sigma_n^{(q)})\right\|^2
    \left\|(\Sigma_\lambda^{(q)})^{-(\frac{1}{2}\vee r -\frac{1}{2})}\Sigma_{N, \lambda}^{\frac{1}{2}\vee r -\frac{1}{2}}\right\|^2 \notag \\
    &\ \ \ \ \ \ \ \ \ \ \times \left\|\Sigma_{N, \lambda}^{-(\frac{1}{2}\vee r -\frac{1}{2})}\Sigma_\lambda^{\frac{1}{2}\vee r -\frac{1}{2}}\right\|^2
    \left\|\Sigma_\lambda^{-(\frac{1}{2}\vee r -\frac{1}{2})}f_t\right\|_H^2 \label{ineq: decomposed_two_errors2}
\end{align}
for any $\lambda >0$. We particularly set $\lambda = 1/(\eta t)$. \par
First, we consider inequality (\ref{ineq: decomposed_two_errors2}) which corresponds to the second term. For bounding $\left\|\Sigma^{\frac{1}{2}-s} \Sigma_{N, \lambda}^{-\frac{1}{2}}\right\|^2$ and $ \left\|\Sigma_{N, \lambda}^{-(\frac{1}{2}\vee r -\frac{1}{2})}\Sigma_\lambda^{\frac{1}{2}\vee r -\frac{1}{2}}\right\|^2 \leq \left\|\Sigma_{N, \lambda}^{-\frac{1}{2}}\Sigma_\lambda^{\frac{1}{2}}\right\|^{4(\frac{1}{2}\vee r -\frac{1}{2})}$ (from Cordes Inequality, Proposition 4 in \cite{rudi2017generalization}), we apply Lemma \ref{lem: emp_exp_prod}. Similarly, we can bound $\left\|\Sigma_{N, \lambda}^\frac{1}{2}(\Sigma_\lambda^{(q)})^{-\frac{1}{2}}\right\|^2$ and $\left\|\Sigma_{N, \lambda}^{-(\frac{1}{2}\vee r -\frac{1}{2})}\Sigma_\lambda^{\frac{1}{2}\vee r -\frac{1}{2}}\right\|^2$ using Lemma \ref{lem: is_emp_exp_prod}. Also, we can use Lemma \ref{lem: spectral_filters}$ for  \left\|(\Sigma_\lambda^{(q)})^{-(\frac{1}{2}\vee r -\frac{1}{2})}\Sigma_{N, \lambda}^{\frac{1}{2}\vee r -\frac{1}{2}}\right\|^2$. Finally, $\left\|\Sigma_\lambda^{-(\frac{1}{2}\vee r -\frac{1}{2})}f_t\right\|_H^2$ can be bounded by Lemma \ref{lem: f_boundedness}. \par
Next we focus on inequality (\ref{ineq: decomposed_two_errors1}). We can use Lemma \ref{lem: emp_exp_prod} for bounding  $\left\|\Sigma^{\frac{1}{2}-s}\Sigma_{N, \lambda}^{-\frac{1}{2}}\right\|^2$ and $\left\|\Sigma_{N, \lambda}^{-\frac{1}{2}}\Sigma_\lambda^{\frac{1}{2}}\right\|^2 \leq 1 + \left\|\Sigma^{\frac{1}{2}-s}\Sigma_{N, \lambda}^{-\frac{1}{2}}\right\|^2$ with high probability. Also, for bounding $\left\|\Sigma_{N,\lambda}^{\frac{1}{2}}(\Sigma_{n, \lambda}^{(q)})^{-\frac{1}{2}}\right\|^2 = \left\|\Sigma_{N,\lambda}^{\frac{1}{2}}(\Sigma_{n, \lambda}^{(q)})^{-\frac{1}{2}}\right\|^2\Sigma_{\frac{1}{\eta t}}^{-\frac{1}{2}}(\Sigma f_t - S^*f_*)$, Lemma \ref{lem: is_grad_diff} and \ref{lem: is_emp_exp_prod} can be applied. For bounding $\left\|\Sigma_\lambda^{-\frac{1}{2}}((S_n^{(q)})^*\bm y_n - \Sigma_n^{(q)}f_t)\right\|_H^2$, note that the decomposition $\left\|\Sigma_\lambda^{-\frac{1}{2}}((S_n^{(q)})^*\bm y_n - \Sigma_n^{(q)}f_t)\right\|_H^2 \leq 2\left\|\Sigma_\lambda^{-\frac{1}{2}}((S_n^{(q)})^*\bm y_n - \Sigma_n^{(q)}f_t)- \Sigma_{\lambda}^{-\frac{1}{2}}(\Sigma f_t - S^*f_*)\right\|_H^2 + 2\left\|\Sigma_{\lambda}^{-\frac{1}{2}}(\Sigma f_t - S^*f_*)\right\|_H^2$. The second term can be bounded by Lemma \ref{lem: excess_diff}. Also, $\left\|(\Sigma_{n, \lambda}^{(q)})p_t(\Sigma_n^{(q)})\right\|^2$ can be bounded by Lemma \ref{lem: spectral_filters}. Then, we set $A$ to the event that all the aforementioned bounds hold on with high probability and apply Lemma \ref{lem: is_grad_diff}. Combining the results leads to the following proposition: 
\begin{proposition}\label{prop: cred_il}
Suppose that $\eta = O(1/\kappa^2)$ be sufficiently small. Let $t \in \mathbb{N}$, $\lambda =1/(\eta t) \geq \lambda_q = \Omega((\mathrm{Tr}(\Sigma^{1/\alpha})/n)^\alpha)$, $\delta \in (0, 1)$ and $n \geq \widetilde \Omega(1+\mathrm{Tr}(\Sigma^{1/\alpha})\lambda_q^{-1/\alpha})$ and $N \geq \widetilde \Omega(1+\kappa^2\lambda_q^{-1})$. Then there exists event $A$ with $P(A) \geq 1 - \delta$ such that
\begin{align*}
    \mathbb{E}\left[\|g_t - f_t\|_{L^2(\rho_{\mathcal X})}^2 \mid A\right] 
    = \widetilde O\left( \frac{\mathrm{Tr}(\Sigma^{\frac{1}{\alpha}})(\sigma^2+R^2\lambda^{2r})\lambda_q^{-\frac{1}{\alpha}}}{n} + \lambda^{2r} + 
    r_N\right),
\end{align*}
where
$$r_N = \frac{\kappa^2\lambda_q^{-1}}{nN}\left(\sigma^2 + R^2\lambda^{2r} + \frac{M^2+\kappa^{4r-2}R^2+R^2\lambda^{-1+2r}}{N}\right).$$
Here $\widetilde{O}$ hides extra $\mathrm{poly}(\mathrm{log}(n), \mathrm{\delta^{-1}}))$ factors.
\end{proposition}

\section{Equivalence of Gradient Descent Solution to Analytic Solution}\label{app_sec: equivalence}
Let $g_{\lambda}' = (\Sigma_{n, \lambda}^{(q)})^{-1}(S_n^{(q)})^*\bm y_n \in H$. We want to bound $\|S(g_{\lambda_*}' -  f_t)\|_{L^2(\rho_{\mathcal X})}^2$ for $\eta = \Theta(1/\kappa^2)$, where $\lambda_*$ is defined in Definition \ref{def: opt_num_iter} in the main paper.
First we decompose the error $\|\Sigma^{1/2-s}(g_{\lambda_*}' - f_{t})\|_H^2$ to two terms: 
\begin{align*}
    &\left\|\Sigma^{\frac{1}{2}-s}(g_{\lambda_*}' - f_{t})\right\|_H^2 \\
    \leq&\ 2\left\|\Sigma^{\frac{1}{2}-s}(g_t - (\Sigma_{n, \lambda_*}^{(q)
    })^{-1}\Sigma_n^{(q)}f_{t})\right\|_H^2
    + 2\left\|\Sigma^{\frac{1}{2}-s}(\Sigma_{n, \lambda_*}^{(q)
    })^{-1}\Sigma_n^{(q)}f_{t} - f_{t})\right\|_H^2. 
\end{align*}

The first term can be bounded as follows:
\begin{align}
    &\left\|\Sigma^{\frac{1}{2}-s}(g_{\lambda_*}' - (\Sigma_{n, \lambda_*}^{(q)
    })^{-1}\Sigma_n^{(q)}f_{t})\right\|_H^2 \notag \\
    =&\ \left\|\Sigma^{\frac{1}{2}-s}(\Sigma_{n, \lambda_*}^{(q)
    })^{-1}(S_n^{(q)})^*\bm y_n - (\Sigma_{n, \lambda_*}^{(q)
    })^{-1}\Sigma_n^{(q)}f_{t})\right\|_H^2 \notag \\
    =&\ \left\|\Sigma^{\frac{1}{2}-s}\Sigma_{N, \lambda_*}^{-\frac{1}{2}}\right\|^2
    \left\|\Sigma_{N,\lambda_*}^{\frac{1}{2}}(\Sigma_{n, \lambda_*}^{(q)})^{-\frac{1}{2}}\right\|^2
    \left\|(\Sigma_{n, \lambda_*}^{(q)})^{-\frac{1}{2}}\Sigma_{N, \lambda_*}^\frac{1}{2}\right\|^2 \notag \\
    &\ \ \ \ \ \ \ \ \ \ \ \times \left\|\Sigma_{N, \lambda_*}^{-\frac{1}{2}}\Sigma_{\lambda_*}^{\frac{1}{2}}\right\|^2
    \left\|\Sigma_{\lambda_*}^{-\frac{1}{2}}((S_n^{(q)})^*\bm y_n - \Sigma_n^{(q)}f_{t})\right\|_H^2. \label{ineq: decomposed_two_errors3}
\end{align}
The second term has following bound:
\begin{align}
    &\left\|\Sigma^{\frac{1}{2}-s}((\Sigma_{n, \lambda_*}^{(q)
    })^{-1}\Sigma_n^{(q)}f_t - f_t)\right\|_H^2 \notag \\
    =&\ \lambda_*^2\left\|\Sigma^{\frac{1}{2}-s}(\Sigma_{n, \lambda_*}^{(q)})^{-1}f_t\right\|_H^2 \notag \\
    \leq&\ \lambda_*^2\left\|\Sigma^{\frac{1}{2}-s} \Sigma_{N, \lambda_*}^{-\frac{1}{2}}\right\|^2
    \left\|\Sigma_{N, \lambda}^\frac{1}{2}(\Sigma_{\lambda_*}^{(q)})^{-\frac{1}{2}}\right\|^2\left\|(\Sigma_{\lambda_*}^{(q)})^{-(1 - \frac{1}{2}\vee r)}\right\|^2
    \left\|(\Sigma_{\lambda_*}^{(q)})^{-(\frac{1}{2}\vee r -\frac{1}{2})}\Sigma_{N, \lambda}^{\frac{1}{2}\vee r -\frac{1}{2}}\right\|^2 \notag \\
    &\ \ \ \ \ \ \ \ \ \ \times \left\|\Sigma_{N, \lambda}^{-(\frac{1}{2}\vee r -\frac{1}{2})}\Sigma_\lambda^{\frac{1}{2}\vee r -\frac{1}{2}}\right\|^2
    \left\|\Sigma_\lambda^{-(\frac{1}{2}\vee r -\frac{1}{2})}f_t\right\|_H^2 \notag \\
    \leq&\ \lambda_*^{1\vee 2r}\left\|\Sigma^{\frac{1}{2}-s} \Sigma_{N, \lambda_*}^{-\frac{1}{2}}\right\|^2
    \left\|\Sigma_{N, \lambda}^\frac{1}{2}(\Sigma_{\lambda_*}^{(q)})^{-\frac{1}{2}}\right\|^2
    \left\|(\Sigma_{\lambda_*}^{(q)})^{-(\frac{1}{2}\vee r -\frac{1}{2})}\Sigma_{N, \lambda}^{\frac{1}{2}\vee r -\frac{1}{2}}\right\|^2 \notag \\
    &\ \ \ \ \ \ \ \ \ \ \times \left\|\Sigma_{N, \lambda}^{-(\frac{1}{2}\vee r -\frac{1}{2})}\Sigma_\lambda^{\frac{1}{2}\vee r -\frac{1}{2}}\right\|^2
    \left\|\Sigma_\lambda^{-(\frac{1}{2}\vee r -\frac{1}{2})}f_t\right\|_H^2. \label{ineq: decomposed_two_errors4}
\end{align}
We particularly set $t = 1/(\eta \lambda_*)$. The only differences from the arguments in Section \ref{app_sec: main_results} are the replacements of $\left\|(\Sigma_{n, \lambda}^{(q)})p_t(\Sigma_n^{(q)})\right\|^2$ (which has bound $O(1)$) with $1$ and $\left\|(\Sigma_\lambda^{(q)})^{\frac{1}{2} \vee r} r_t(\Sigma_n^{(q)})\right\|^2$ (which has a bound $O(\lambda^{1\vee 2r})$) with $\lambda^{1\vee 2r}$. Hence, we obtain the perfectly same variance bound as the one of gradient descent in Theorem \ref{prop: cred_il}.

\section{Sufficient Condition for $\mathcal N_\infty(\lambda) \ll \mathcal F_\infty(\lambda)$}\label{app_sec: sufficient_condition}
\begin{proposition}\label{prop: F_lower_bound}
Let $\{(\lambda_i, \phi_i)\}_{i=1}^d$ ($d \in \mathbb{N}\cup\{\infty\}$) be the eigen-system of $\Sigma$ in $L^2(\rho_{\mathcal X})$, where $\lambda_1 \geq \lambda_2 \geq \ldots > 0$. Assume that $\lambda_i = \Theta(i^{-\alpha})$ and  $\|\phi_i\|_{L^\infty(\rho_{\mathcal X})} = \Omega(i^{p/2})$ for any $i$ for some $\alpha = 1 + \Omega(1)$ and $p \geq 1$. Moreover if $d = \infty$, we additionally assume $\|\phi_i\|_{L^\infty(\rho_{\mathcal X})}^2 = O(i^{\alpha-1-\varepsilon})$ for any $i$ for some $\varepsilon > 0$. Then Assumption \ref{assump: kernel_boundedness} is satisfied and for any $\lambda \in (0, 1)$, 
\begin{align*}
    \mathcal F_\infty(\lambda) = \Omega\left(\lambda^{-\frac{p}{\alpha}}\wedge d^p\right).
\end{align*}
\end{proposition}
\begin{proof}
First note that from Mercer's theorem, we have  $K(x, x') = \sum_{i=1}^d \lambda_i \phi_i(x)\phi_i(x')$. Assumption \ref{assump: kernel_boundedness} is always satisfied when $d < \infty$ and thus we consider the case $d = \infty$. Since $\|\phi_i\|_{L^\infty(\rho_{\mathcal X})}^2 = O(i^{\alpha - 1 -\varepsilon})$, $\|K_x\|_H^2$ is uniformly bounded and thus Assumption \ref{assump: kernel_boundedness} is satisfied. Let $\lambda > 0$.
\begin{align*}
    \mathcal F_\infty(\lambda) =&\ \mathrm{sup}_{x\in\mathrm{supp}(\rho_{\mathcal X})}\|\Sigma_\lambda^{-\frac{1}{2}}K_x\|_H^2 \\
    =&\ \mathrm{sup}_{x\in\mathrm{supp}(\rho_{\mathcal X})}\left\|\sum_{i=1}^d \lambda_i \phi_i(x) (\lambda_i+\lambda)^{-\frac{1}{2}}\phi_i\right\|_H^2 \\
    =&\ \mathrm{sup}_{x\in\mathrm{supp}(\rho_{\mathcal X})}\sum_{i=1}^d \frac{\lambda_i\phi_i(x)^2}{\lambda_i + \lambda} \\
    \geq&\  \frac{\lceil  \lambda^{-\frac{1}{\alpha}}\wedge d \rceil^{-\alpha + p}}{\lceil  \lambda^{-\frac{1}{\alpha}}\wedge d \rceil^{-\alpha}+\lambda} \\
    =&  \Omega\left(\lambda^{-\frac{p}{\alpha}}\wedge d^p\right) 
\end{align*}
\end{proof}

\section{Extension to Random Features Settings}\label{app_sec: random_features}
The following lemma is analogous to Lemma \ref{lem: f_boundedness}. 
\begin{lemma}\label{lem: rf_f_boundedness}
Suppose that Assumptions \ref{assump: smoothness_of_true} and \ref{assump: rf_boundedness} hold.
Let $\eta = O(1/\kappa^2)$ be sufficiently small and $t \in \mathbb{N}$ such that $m = \widetilde \Omega(1+\kappa^2\eta t)$. Then for any $\delta > 0$, with probability at least $1 - \delta$,   
\begin{align*}
    \left\|\hat S \hat f_t - f_*\right\|_{L^2(\rho_\mathcal X)}^2 = O\left(R^2(\eta t)^{-2r}\right)
\end{align*}
and for any $\lambda > 0$ and $s \in [0, r]$
\begin{align*}
    \|\Sigma_\lambda^{-s}f_t\|_H^2 = O(R^2(\kappa^{4(r-s) - 2}+(\eta t)^{1-2(r-s)})).
\end{align*}
\begin{proof}
Recall that $\hat f_t = \hat f_{t-1} - \eta (\hat \Sigma \hat f_{t-1} - \hat S^*f_*)$ and $\hat f_0 = 0$. Thus we have $\hat S \hat f_t = \hat S \hat f_{t-1} - \eta \hat {\mathcal L}(\hat S \hat f_{t-1} - f_*)$. Hence it holds that
\begin{align*}
    \hat S \hat f_t - f_* = (I - \eta \hat{\mathcal L})(\hat S f_t - f_*).
\end{align*}
Therefore we get
\begin{align*}
    \|\hat S \hat f_t - f_*\|_{L^2(\rho {\mathcal X})}^2 = \|(I - \eta \hat{\mathcal L})f_*\|_{L^2(\rho {\mathcal X})}^2 = \|(I - \eta \hat {\mathcal L})^t\mathcal L^r \phi\|_{L^2(\rho {\mathcal X})}^2 = O(R^2)\|(I - \eta \hat{\mathcal L})^t\mathcal L^r\|^2.
\end{align*}
Let $\lambda' > 0$. Observe that $\|(I - \eta \hat{\mathcal L})^t\mathcal L^r\|^2 \leq \|(I - \eta \hat{\mathcal L})^t \hat {\mathcal L}_{\lambda'}^r {\mathcal L}_{\lambda'}^{-r}\mathcal L^r\|^2 
\leq \|(I - \eta \hat{\mathcal L})^t \hat {\mathcal L}_{\lambda'}^r\|^2\|\hat{\mathcal L}_{\lambda'}^{-r}\mathcal L^r\|^2$. 
We have $\|(I - \eta \hat{\mathcal L})^t \hat {\mathcal L}_{\lambda'}^r\|^2 \leq \|(I - \eta \hat{\mathcal L})^t \hat {\mathcal L}^r\|^2 + {\lambda'}^{2r} = O((\eta t)^{2r} + {\lambda'}^{2r})$ from Lemma \ref{lem: spectral_filters}. Also similar to Lemma \ref{lem: emp_exp_prod} with $s=0$, with high probability we have $\|\hat{\mathcal L}_{\lambda'}^{-r}\mathcal L^r\|^2 = O(1)$ if $m = \widetilde \Omega(1+\kappa^2\lambda'^{-1})$. Finally setting $\lambda' = (\eta t)^{-1}$ yields the first statement. The second statement can be easily proven in a very similar manner to the proof of Lemma 16 in \cite{lin2017optimal} but we need to use the fact that $\|\hat{\mathcal L}_{\lambda'}^{-r}\mathcal L^r\|^2 = O(1)$ with high probability as in the proof of the first statement. This finishes the proof.
\end{proof}
\end{lemma}
\begin{lemma}[Proposition 10 in \citep{rudi2017generalization}] \label{lem: rf_N_diff_bound} Suppose that Assumption \ref{assump: rf_boundedness} holds.
We denote 
$\hat{\mathcal N}_\infty(\lambda) = \mathbb{E}_x\|\hat \Sigma_\lambda^{-1/2} \phi_{m, x}\|_2^2$ for $\lambda > 0$.
For any $\delta \in (0, 1)$ and sufficiently small $\lambda = O(1)$, if $m = \widetilde \Omega (1 + \kappa^2\lambda^{-1})$, with probability at least $1-\delta$ it holds that
\begin{align*}
    \hat{\mathcal N}_\infty(\lambda) \leq 1.55 \mathcal N_\infty(\lambda).
\end{align*}
\end{lemma}
Combining the bias and variance bounds with Lemma \ref{lem: rf_N_diff_bound} yields the following theorem:
\begin{theorem}[Generalization Error of CRED-GD with RF]\label{thm: rf_cred_il}
Suppose that Assumptions \ref{assump: smoothness_of_true}, \ref{assump: eigen_decay}, \ref{assump: noise} and \ref{assump: rf_boundedness} hold. Let $\eta = \Theta(1/\kappa^2)$ be sufficiently small, $\lambda_q = \lambda_*$ and $T = \widetilde \Theta (t_\eta^*)$. For any $\delta \in (0, 1)$, if $m \geq \widetilde O(1 + \kappa^2\lambda_*^{-1})$, there exists event $A$ with $P(A) \geq 1 - \delta$ such that RF-CRED-GD satisfies
\begin{align*}
    \mathbb{E}\left[\left\|\hat S \hat g_T - f_*\right\|_{L^2(\rho_{\mathcal X})}^2 \mid A \right] 
    = \widetilde{O}\left(
    \left(\frac{\sigma^2\mathrm{Tr}(\Sigma^{\frac{1}{\alpha}})}{n}\right)^{\frac{2r\alpha}{2r\alpha+1}}
    +    
    \left(\frac{R^2\mathrm{Tr}(\Sigma^{\frac{1}{\alpha}})}{n}\right)^{2r\alpha} 
    +
    \lambda_N^{2r}
    \right),
\end{align*}
where $\lambda_N$ is defined in Definition \ref{def: opt_num_iter} in Section \ref{sec: generalization} of the main paper.
\end{theorem}


\end{document}